\documentclass[12pt]{article}
\usepackage{authblk}
\usepackage[utf8]{inputenc}
\usepackage[T1]{fontenc}
\usepackage{amsmath,amssymb,amsthm}
\usepackage{mathtools}
\usepackage{mathrsfs}
\usepackage{bm}
\usepackage{hyperref}
\hypersetup{colorlinks=true, linkcolor=blue}
\usepackage{geometry}
\usepackage{graphicx}
\usepackage{float}
\usepackage{algorithm}
\usepackage{algpseudocode}
\usepackage{enumerate}
\usepackage[breakable, skins]{tcolorbox}
\usepackage{placeins}
\usepackage[backend=biber,style=alphabetic]{biblatex}
\usepackage{textcomp}
\usepackage{hyphenat}

\usepackage{appendix}
\usepackage{indentfirst,csquotes}
\usepackage{titlesec}
\usepackage[english]{babel}
\usepackage{xcolor,paralist,fancyhdr,etoolbox}
\usepackage{algorithm}
\usepackage{algpseudocode}
\usepackage{float}
\usepackage{tcolorbox}
\usepackage{booktabs}
\usepackage[mathscr]{euscript}
\usepackage{mathrsfs, mathtools, stmaryrd}
\usepackage{geometry}
\usepackage{multirow}
\usepackage{stmaryrd}
\usepackage{cleveref}

\newtheorem{theorem}{Theorem}[section]
\newtheorem{lemma}[theorem]{Lemma}
\newtheorem{definition}[theorem]{Definition}

\newtheorem{corollary}[theorem]{Corollary}

\newcommand{\E}{\mathbb{E}}
\newcommand{\R}{\mathbb{R}}
\newcommand{\norm}[1]{\left\|#1\right\|}
\newcommand{\ip}[2]{\langle #1,\, #2 \rangle}

\newcommand{\Id}{I}
\newcommand{\Sig}{\Sigma}

\newcommand{\I}{{I}}

\renewcommand{\Pr}{\mathbb{P}}

\newcommand{\Sigmax}{\Sigma_x}

\definecolor{algoheader}{RGB}{50,50,50}
\tcbset{
  mainthm/.style={
    breakable,
    colback=blue!4!white,
    colframe=blue!55!black,
    fonttitle=\bfseries,
    before upper={\setlength{\parskip}{4pt}},
    left=6pt, right=6pt, top=4pt, bottom=4pt
  }
}

\title{ Modulated learning for private and distributed regression with just a single sample per client device}
\date{}
\author{Praneeth Vepakomma$^{1,2}$, Amirhossein Reisizadeh$^{1}$,\\ Samuel Horváth$^{2}$, Munther Dahleh}
\affil[1]{Massachusetts Institute of Technology (MIT)}
\affil[2]{Mohamed bin Zayed University of Artificial Intelligence (MBZUAI)}

\begin{document}

\maketitle


\begin{abstract}
This work focuses on the question of learning from a large number of devices with each device holding only a single sample of data. Several real-world applications exist to this one sample per client setup up including learning from fitness trackers, data/app usage aggregators, body-worn sensing devices, and daily event monitors to name a few. When a client has only one sample, the standard federated learning paradigm breaks down as a local update based on that single point is far from being useful, especially in the earlier rounds for estimation of the model coefficients. This utility is further weakened by the privacy-inducing noise applied at every round. This work caters to this problem to enable such clients to collaboratively contribute to effectively learn a global model without leaking the privacy of their data. The proposed approach injects a single, carefully calibrated noisy perturbation to transform the sample at each client, followed by a post-processed representation which is shared with the server. These representations aggregated at the server are processed to obtain an unbiased gradient update that in expectation matches the non-private centralized gradient while preserving data privacy. This approach is different than traditional private federated learning, where the communication payloads involve model coefficients as opposed to privately transformed data samples. This method enables devices with extremely limited data to collaborate and learn accurate, privacy-preserving models without requiring large local datasets or sacrificing individual privacy.
\end{abstract}




\section{Introduction} Privacy-preserving federated learning typically assumes that each client holds a local dataset large enough to compute meaningful updates over multiple gradient steps. However, in many real-world applications (mobile health apps, fitness trackers, body worn sensing, or daily event monitors) each single client may hold only a few or in the extreme, often just a single data point. In such regimes, existing mechanisms like differentially private stochastic gradient descent (DP-SGD) based federated learning schemes become substantially ineffective as local updates are noisy and unreliable as they are learned from a single sample. On top of this, the privacy-inducing noise injected into this system further exacerbates this problem.

This paper thereby asks: \emph{can we enable useful, differentially private distributed learning when each client contributes only one example data point?} We address this challenge by designing a cosine-modulated, contractive transformation followed by a privacy-preserving Gaussian perturbation. These transformed data points are then further processed with a specific functional form to create an intermediate result that assists the server in estimating the gradient of the learning loss. The scheme enables the server to obtain an unbiased gradient update that in expectation matches the non-private centralized gradient update, while still preserving data privacy. This approach is different from traditional private federated learning, where the communication payloads involve model coefficients as opposed to privately transformed data samples.
Building upon the unbiasedness of our estimator, we further derive the covariance of the stochastic gradient updates and establish the asymptotic normality of our model estimator. This facilitates asymptotically valid inference, such as hypothesis testing, on the underlying regression coefficients within our distributed framework.
The remainder of the paper is organized as follows. Section 2 fixes the privacy definitions and background needed for the method. Section 3 develops the core single-vector modulated protocol and shows how the server recovers an unbiased gradient estimator from privatized client messages. Sections 4–6 present a non-iterative variant, a reconstruction lower bound, and a multi-vector refinement. Sections 7–9 analyze variance, convergence, and design trade-offs, and the appendices collect proofs, experiments, and supporting background.

\section{Preliminaries}Before introducing the protocol itself, we fix the privacy and learning notions used throughout the paper. This section provides the minimum background needed to state the client mechanism and server correction cleanly.
In this section, we list the relevant preliminaries and related works on differential privacy, private linear regression and some existing paradigms for collaborative machine learning, before presenting the proposed methods.
\subsection{Differential Privacy}
Because privacy in the proposed protocol is enforced through the client-side transformation and its sensitivity, we begin with the relevant local-DP definitions. The main takeaway of this subsection is that once the client map has a sensitivity bound, the Gaussian mechanism and post-processing properties maintain the privacy guarantee used throughout the paper.
\begin{definition}
  (Differential privacy \cite{dwork2006differential,dwork2014algorithmic}) A randomized algorithm often called a mechanism $
\mathcal{M}: \mathcal{X} \rightarrow \mathcal{Y}
$ satisfies ${\varepsilon}$-local differential privacy (or ${\varepsilon}$- local DP) if, for every pair of possible inputs $x, x^{\prime} \in \mathcal{X}$ and every possible set of outputs $S \subseteq \mathcal{Y}$, we have

$$
\operatorname{Pr}[\mathcal{M}(x) \in S] \leq e^{\varepsilon} \cdot \operatorname{Pr}\left[\mathcal{M}\left(x^{\prime}\right) \in S\right] +\delta.
$$
Here, $\varepsilon$ controls the privacy leakage (smaller being more private), and $\delta$ captures a small probability of failure, where DP does not hold. When $\delta = 0$, this reduces to pure DP; otherwise, it is referred to as approximate DP. The goal of DP is to ensure that the inclusion or exclusion of a single individual’s data does not significantly affect the output of an mechanism, thereby protecting privacy. \end{definition} We now describe certain properties of DP that are quite useful. 
Differential privacy is closed under postprocessing as any function applied to the output of a DP mechanism does not degrade the privacy guarantee.\begin{theorem}[Immunity to postprocessing]
If $\mathcal{M}$ is $(\varepsilon, \delta)$-DP and $g$ is any (randomized or deterministic) function, then $g \circ \mathcal{M}$ is also $(\varepsilon, \delta)$-DP.
\end{theorem}
This essentially means that differential privacy is preserved under post-processing, and so any deterministic computation on the privatized outputs does not degrade the privacy guarantee.\par A classical way to obtain a differentially private mechanism for real-valued queries is to use carefully callibrated additive noise from a Gaussian distribution (for approximate DP) or a Laplace distribution (for pure DP). The amount of noise needed, depends on the notion of a global sensitivity of the query, that we define below.
\begin{definition} (Global sensitivity)
Global sensitivity is the maximum change in the output of a function $\Delta M_G$ when applied to any two neighboring datasets $D,D^{\prime}\in \mathcal{D}$ (or records in a special case), differing by a single element (or by a small amount in any chosen norm) as $$
\Delta M_G=\max _{D, D^{\prime}\in \mathcal{D}}\left\|M(D)-M(D^{\prime})\right\|_p
$$Hence, it is independent of any specific dataset and is instead a query-dependent measure of the maximum change in the output of a function $M$ over all possible neighboring datasets.
\end{definition} That said,
several other equivalent definitions (or different variants) of $(\varepsilon, \delta)$-differentially privacy also exist in the literature, some of which are detailed in Appendix \ref{appEquiv}.
\subsection{Collaborative learning}
Popular collaborative machine learning paradigms \cite{ben2019demystifying} include federated learning \cite
{konevcny2016federated,kairouz2021advances,mcmahan2017communication,bonawitz2019towards}, split learning \cite{vepakomma2018split,gupta2018distributed,thapa2022splitfed}, and local parallelism \cite{laskin2020parallel} where client devices hold local data, and only intermediate computations such as model gradients or intermediate representations such as activations are shared by them with a central server. These model updates (gradients or intermediate representations) from each client are aggregated at the server to improve the global model over multiple such rounds of interaction. The clients then download the current model from central server and continue to improve it locally by further updating their model weights based on their local data and are again sent back to the server for aggregation and this process continues. There is no explicit sharing of raw data in this setup and other such methods include \cite{zhang2017poseidon,assran2019stochastic, sergeev2018horovod, gu2023elasticflow}.
\subsection{Linear regression}
In linear regression, the relationship between a response $y \in \mathbb{R}$ and covariates $x \in \mathbb{R}^d$ is given by
$$
y = x^\top \beta + \epsilon, \quad \epsilon \sim \mathcal{N}(0, \sigma^2).
$$
Given $n$ i.i.d. samples $\{(x_i, y_i)\}_{i=1}^n$ from this model, with ${X} \in \mathbb{R}^{n \times d}$ denoting the design matrix and ${Y}$ denoting the response, the ordinary least squares (OLS) estimator of $\beta$ is
$$
\hat{\beta}_{\text{OLS}} = ({X}^\top {X})^{-1} {X}^\top {Y},
$$
and defined when ${X}$ has full column rank. Under Gaussian errors and fixed design, OLS is minimax optimal and achieves the Cramér-Rao lower bound, with the minimax risks as detailed in Appendix \ref{LRapp}.
\section{One (or few) private sample(s) per client} With the preliminaries in place, we now turn to the main regime of interest: many clients, each holding only private one sample. The goal of this section is to show that even in this extreme setting, privatized client messages can be designed so that the server recovers an unbiased centralized gradient in expectation. \par We provide two variants of the protocol with the first being the \textit{single vector modulated} private FL mechanism where the sample on each client is privately transformed using a proposed map while ensuring $(\varepsilon,\delta)$-differential privacy. This is then released for post-processing at the server, to achieve predictive utility in a privacy-preserving manner. 
The post-processing is designed so that the server's aggregate of transformed data representations received from all the clients helps recover an unbiased estimate of the \textit{centralized non-private gradient} in expectation, although while preserving privacy. We generalize this core methodology to a multi-vector framework, utilizing multiple orthonormal vectors to modulate the privacy inducing transformation of our sample, and rigorously establish its efficacy in reducing the mean squared error beyond the original formulation. 

\subsection{Modulated learning}
 We first develop the basic modulated-learning construction that underlies the rest of the paper. The plan is to define the client-side map, quantify its sensitivity, and then show how the server debiases the received statistics to obtain a usable gradient estimate.
 We start with the definition of the prescribed modulating map and correspondingly explain the role of various parameters used in this mapping. 
\subsubsection{Single vector modulated DP}\label{singVec}

At round $t$ of the protocol, the server maintains the global model $\beta_t$ and selects a modulation direction $v_t \in \mathbb{R}^d$ with $\|v_t\|_2 = 1$ and $v_t \perp \beta_t$. The tuple $(\beta_t, v_t)$, along with chosen parameters $(\alpha, \lambda, \omega)$ and the privacy inducing noise level $\sigma$ that is needed, is broadcast to all the clients. Upon receiving $(\beta_t,v_t)$, client $i$ draws $\phi_i^{(t)} \sim \mathrm{Unif}[0, 2\pi]$ and computes the modulated transformation given below.
\begin{equation} \label{eq:map}
g_i^{(t)} = (1 - \alpha) x_i + \lambda \cos\big( \omega \langle x_i, v_t\rangle + \phi_i^{(t)} \big) v_t.
\end{equation}
The parameters of $\alpha,\lambda$ and $\omega$ impact the global sensitivity  of the transformation $g(\cdot)$, which will allow us to calibrate the variance of the additive noise $\sigma^2$ that needs to be used to induce differential privacy. This result is formalized in the lemma~\ref{thm:lipschitz_placeholder} below.
\subsubsection{Role of \texorpdfstring{$\alpha,\lambda$ and $v$}{alpha, lambda and v}}

\begin{lemma}[Lipschitz continuity of $g$]\label{thm:lipschitz_placeholder}
Let $0<\alpha<1$, $\lambda>0$, $\omega>0$, $v\in\mathbb{R}^{d}$ with $\norm{v}_2=1$ and $\phi_0\in[0,2\pi]$.
Define
\begin{equation}
    g(x) \;=\; (1-\alpha)\,x \;+\; \lambda\,\cos\bigl(\omega\,\ip{x}{v}+\phi_0\bigr)\,v + \eta_0,
\quad \text{where } \eta_0 \sim \mathcal{N}(0, \sigma^2 I_d)
\end{equation}

Then $g$ devoid of the noise $\eta_0$ is globally Lipschitz with constant
\[
L \;=\; \lvert 1-\alpha\rvert + \lambda\,\omega .
\]
That is, $\norm{g(x)-g(y)}_2 \,\le\, L\,\norm{x-y}_2$ for all $x,y\in\mathbb{R}^{d}$.
\end{lemma}
\begin{proof}
Decompose $g(x)-g(y)$:
\[
g(x)-g(y)
=(1-\alpha)(x-y)
+\lambda\bigl[\cos(\omega\ip{x}{v}+\phi_0)-\cos(\omega\ip{y}{v}+\phi_0)\bigr]\,v .
\]
Taking $\ell_2$–norms and using $\norm{v}_2=1$ together with the triangle inequality,
\[
\norm{g(x)-g(y)}_2
\le
\lvert 1-\alpha\rvert\,\norm{x-y}_2
+
\lambda\,
\bigl|\cos(\omega\ip{x}{v}+\phi_0)-\cos(\omega\ip{y}{v}+\phi_0)\bigr| .
\]
Since $\cos t$ is $1$–Lipschitz,  
$|\cos u-\cos w|\le|u-w|$ for all $u,w\in\mathbb{R}$.  
With $u=\omega\ip{x}{v}+\phi_0$ and $w=\omega\ip{y}{v}+\phi_0$,  
$|u-w|=\omega\,|\ip{x-y}{v}|\le\omega\,\norm{x-y}_2$.  
Substituting gives
\[
\norm{g(x)-g(y)}_2
\le
\bigl(\lvert 1-\alpha\rvert+\lambda\omega\bigr)\,\norm{x-y}_2,
\]
so $L=\lvert 1-\alpha\rvert+\lambda\omega$ is a valid Lipschitz constant, and the global sensitivity under the adjacency relation of $x \sim y$ being neighbors when $\left\|x-y\right\|_2 \leq 1$.
\end{proof}Therefore, there are several choices of these parameter values such that the Lipschitz constant is strictly less than $1$, resulting in a corresponding transformation that is a contractive map.

\subsubsection{Client-side broadcast} This Lipschitz bound (in this case, the global sensitivity) is needed to induce privacy, as it quantifies the maximum change in output $g(x)$ for neighboring inputs. To be more precise for the given $L$, the Gaussian mechanism \cite{dwork2016calibrating} to induces differential privacy, by applying an additive Gaussian noise $\xi_i^{(t)} \sim \mathcal{N}(0, \sigma^2 I_d)$ with
$
\sigma = L \sqrt{2\ln(1.25/\delta)}/\varepsilon
$.
The noised version of this transformed data on each client given by,
$
\tilde{g}_i^{(t)} = g_i^{(t)} + \xi_i^{(t)}
$ is then $(\varepsilon,\delta)$-differentially private. This $\tilde{g}_i$, is then sent to the server for a specific post-processing that enables the regression task. 
The client-side of the mechanism is summarized below.
\begin{algorithm}
\caption{Single-Vector: Client Protocol }\label{algo1}
\begin{algorithmic}[1]
\State \textcolor{white}{\colorbox{algoheader}{\parbox{\dimexpr\linewidth-2\fboxsep}{\bfseries\ Input: Private data $(x_i,y_i)$, $(\alpha,\lambda,\omega,\sigma_{\text{DP}}^2)$, $v$}}}
\State Sample $\phi_i \sim \text{Unif}[0,2\pi]$
\State $g_i \gets (1-\alpha)x_i + \lambda \cos(\omega\langle x_i,v\rangle + \phi_i)v$
\State $\xi_i \sim \mathcal{N}(0,\sigma_{\text{DP}}^2 I_d)$
\State $\tilde{g}_i \gets g_i + \xi_i$
\State \textbf{Send} $(\tilde{g}_i, y_i)$ to the server.
\end{algorithmic}
\end{algorithm}
\subsubsection{Server-side aggregation}
 
The server aggregates these received updates as follows.
\begin{align}
\widetilde{\Sigma}_g &= \frac{1}{K} \sum_{i=1}^K \tilde{g}_i^{(t)} \tilde{g}_i^{(t)\top} .
\end{align}

The server then post-processes this $\widetilde{\Sigma}_g $ to get \begin{equation}\widehat{\Sigma}_x = \frac{1}{(1-\alpha)^2} \left( \widetilde{\Sigma}_g - \frac{\lambda^2}{2} vv^\top - \sigma_{\text{DP}}^2 I_d \right).\end{equation} This quantity enables it to obtain an unbiased estimate of the non-private centralized gradient, while still preserving privacy. This is because the privatized matrix $\widehat{\Sigma}_x$  can be post-processed using $y_i$ (which is considered to be non-sensitive in our setting),  $\tilde{g}_i$ which is already a privatized vector, and  $\beta$ which is a privatized vector in any round as follows \begin{equation}
\begin{aligned}
& Z \leftarrow \frac{1}{1-\alpha} \frac{1}{K} \sum_{i=1}^K y_i \tilde{g}_i \\
& G \leftarrow \widetilde{\Sigma}_x \beta-Z
\end{aligned}
\end{equation} to obtain a $G$ such that in expectation, it forms an unbiased estimate of the gradient of the loss function.
This server side of the mechanism is summarized in the algorithm \ref{algo2} below. We clip the coefficients (projection onto the ball of radius $c$), every round as needed to enable convergence analysis, later on. 
\begin{algorithm}
\caption{Single-Vector: Server Protocol }\label{algo2}
\begin{algorithmic}[1]
\State \textcolor{white}{\colorbox{algoheader}{\parbox{\dimexpr\linewidth-2\fboxsep}{\bfseries\ Input: $\{(\tilde{g}_i, y_i)\}_{i=1}^K$, $\beta$, $(\alpha,\lambda,\sigma_{\text{DP}}^2)$}}}
\State $\widetilde{\Sigma}_g \gets \frac{1}{K}\sum_{i=1}^K \tilde{g}_i\tilde{g}_i^\top$
\State $\widetilde{\Sigma}_x \gets \frac{1}{(1-\alpha)^2} \left( \widetilde{\Sigma}_g - \frac{\lambda^2}{2} vv^\top - \sigma_{\text{DP}}^2 I_d \right)$
\State $Z \gets \frac{1}{1-\alpha} \frac{1}{K} \sum_{i=1}^K y_i \tilde{g}_i$
\State $G \gets \widetilde{\Sigma}_x \beta - Z$
\State $\beta \gets \beta- \eta G$
\State $\beta=\frac{\beta}{\max \left(1,\left\|\beta\right\|_2 / c\right)}$
\State {\textbf{Broadcast} the updated regression coefficients $\beta$ and also send new $v \perp \beta$ to the clients.}
\end{algorithmic}
\end{algorithm}
We now show that this algorithm on the server side does result in an unbiased estimate of the gradient while preserving differential privacy. To do so, we first note the following expansion of the expected value $\E[\widetilde{\Sigma}_g]$. The perturbed feature covariance $\widetilde{\Sigma}_g$ is computed from the noisy modulated features $\tilde{g}_i$ sent by the clients. Its expectation can be decomposed into contributions from the noise-free modulated features $g_i$ and the differential privacy (DP) noise $\xi_i$.

\begin{align}
    \E[\widetilde{\Sigma}_g]
    &=\nonumber
    \frac{1}{K} \sum_{i=1}^K \E[\tilde{g}_i\tilde{g}_i^\top]\\\nonumber
    &=
    \frac{1}{K} \sum_{i=1}^K \E[(g_i + \xi_i)(g_i + \xi_i)^\top]\\
    &=
    \frac{1}{K} \sum_{i=1}^K \E[g_i g_i^\top]
    +
    \sigma^2_{\text{DP}} I.
\end{align}The cross terms vanish because the DP noise $\xi_i$ is zero-mean and independent of $g_i$, leaving only the covariance of $g_i$ and the isotropic noise variance. Now we compute the expectation of the noise-free modulated feature covariance:

\begin{align}
    \frac{1}{K} \sum_{i=1}^K \E[g_i g_i^\top]
    &=\nonumber
    (1-\alpha)^2 \frac{1}{K} \sum_{i=1}^K x_i x_i^\top
    +
    \lambda^2 \frac{1}{K} \sum_{i=1}^K  \E[\cos^2(\omega\langle x_i,v\rangle + \phi_i)]v v^\top \\
    &=
    (1-\alpha)^2 \frac{1}{K} \sum_{i=1}^K x_i x_i^\top
    +
    \frac{\lambda^2}{2} v v^\top
\end{align}The second equality uses the fact that $\mathbb{E}[\cos^2(\theta)] = \frac{1}{2}$ when $\theta$ is uniformly distributed over $[0, 2\pi]$. Thus, the cosine modulation contributes a fixed rank-one term scaled by $\lambda^2/2$.

Putting these together,

\begin{align}
    \E[\widetilde{\Sigma}_g]
    =
    (1-\alpha)^2 \frac{1}{K} \sum_{i=1}^K x_i x_i^\top
    +
    \frac{\lambda^2}{2} v v^\top
    +
    \sigma^2_{\text{DP}} I
\end{align}

or, equivalently, after applying the server’s debiasing step:

\begin{align}
    \E[\widetilde{\Sigma}_x]
    =
    \frac{1}{K} \sum_{i=1}^K x_i x_i^\top
\end{align}since $\Sigma_x = \frac{1}{K} \sum_i x_i x_i^\top$, showing that $\widetilde{\Sigma}_x$ is an unbiased estimator of the true data covariance $\Sigma_x$, with the modulation and noise terms removed via algebraic correction.
Similarly, 

\begin{align}
    \E[Z]
    &=
    \frac{1}{1-\alpha} \frac{1}{K} \sum_{i=1}^K \E[y_i \tilde{g}_i]\\
    &=
    \frac{1}{1-\alpha} \frac{1}{K} \sum_{i=1}^K (1-\alpha) y_i x_i + \lambda \E[\cos(\omega\langle x_i,v\rangle + \phi_i)]y_i v \\
    &=
    \frac{1}{K} \sum_{i=1}^K y_i x_i
\end{align} The cosine term vanishes in expectation because $\mathbb{E}[\cos(\cdot)] = 0$ under uniform phase $\phi_i$. Thus, $Z$ is an unbiased estimator of $\frac{1}{K} \sum_i y_i x_i$. Therefore,\begin{align}
    \E[G]
    =
    \E[\widetilde{\Sigma}_x \beta - Z]
    =
    \frac{1}{K} \sum_{i=1}^K x_i x_i^\top \beta
    -
    \frac{1}{K} \sum_{i=1}^K y_i x_i
    =
    \frac{1}{K} X^\top X \beta - \frac{1}{K} X^\top Y 
\end{align}

\begin{align}
    L(\beta) = \frac{1}{2K} \sum_{i=1}^K (x_i^\top \beta - y_i)^2,
    \quad
    \nabla L(\beta) = \frac{1}{K} X^\top X \beta - \frac{1}{K} X^\top Y = \E[G]
\end{align}

Thus, $G$ is an unbiased estimator for the gradient $\nabla L(\beta)$, ensuring that the server’s update direction is correct in expectation while preserving differential privacy.

That requires the server to have a private estimate of $\Sigmax$, as that is the only term in the bias correction function that depends on the raw sensitive data. In  order to estimate it privately, we thereby show the following relation between $\widetilde{\Sigma}_g$ and $\Sigmax$ where $\widetilde{\Sigma}_g$ is obtained as a post-processing of a privatized quantity $\tilde{g}_i^{(t)}$ while $-\frac{\lambda^2}{2} v v^{\top}-\sigma_{\mathrm{DP}}^2 \I_d$ is already private, as it is data independent and public information known to the server.  
$\widetilde{\Sigma}_g$ can be obtained as a post-processing of $\tilde{g}_i^{(t)}$ as shown in the result below, which provides a private estimator for $\Sigmax$ as a corollary.
\begin{tcolorbox}[colback=cyan!10, colframe=black]
\begin{corollary}[Covariance Estimator]\label{thm:cov-estimator}
The sample covariance of perturbed features is defined as $\widetilde{\Sigma}_g = \frac{1}{K} \sum_{i=1}^K \tilde{g}_i \tilde{g}_i^\top$, and the estimator
\[
\widehat{\Sigma}_x = \frac{1}{(1-\alpha)^2} \left( \widetilde{\Sigma}_g - \frac{\lambda^2}{2} v v^\top - \sigma_{\mathrm{DP}}^2 I_d \right)
\]
converges almost surely to $\Sigma_x$ as $K \to \infty$.
\end{corollary}
\end{tcolorbox}

\section{Privacy accounting}
It is straightforward as follows, unlike DP-SGD's moments accountant.
For round $i$, we achieve a $\rho_i$-zCDP guarantee by adding zero-mean
Gaussian noise with variance
\[
    \sigma_i^2 = \frac{\Delta^2}{2\rho_i},
\]
where $\Delta$ is the sensitivity. Since each client operates on a
\emph{distinct} local dataset, \textbf{parallel composition} applies
within each round, so the privacy parameter $\rho_i$ does not
accumulate across clients. Any subsequent computation on the
released output is covered by \textbf{post-processing}, which leaves
$\rho_i$ unchanged.

Across rounds, the \textbf{sequential composition} theorem for zCDP
gives
\[
    \rho_{1:T} = \sum_{i=1}^{T} \rho_i,
\]
so after rounds $i$ and $i+1$ the cumulative budget is simply
$\rho_{i} + \rho_{i+1}$. This is significantly \emph{tighter} than
composition under $(\varepsilon,\delta)$-DP (even via the Advanced
Composition theorem, which incurs an $O(\sqrt{T})$ overhead in
$\varepsilon$), and also tighter than R\'enyi-DP composition, which
requires a subsequent conversion step back to
$(\varepsilon,\delta)$-DP that introduces additional slack.
\section{Non-iterative scheme}
The previous section develops an iterative gradient-based protocol. We now show that the same modulated representation can also support a one-shot estimator.
We now present our non-iterative estimator for the linear regression parameter $\beta^*$. The key insight is that we can leverage the modulated features $h_i$ and responses $y_i$ to construct unbiased moment estimators without requiring multiple rounds of iteration, where $h_i$ are computed as follows:
\[
h_i = (1-\alpha)x_i + \lambda \cos(\omega\langle x_i, v\rangle + \phi_i)v + \eta_i.
\]
This approach provides significant advantages in terms of not having to perform privacy composition across iterations. To obtain the estimator, we first compute the modulated feature-response products, $z_i = h_i y_i$. These products capture the covariance structure between the modulated features and responses. Next, we average these products across all samples
\[
\bar{z} = \frac{1}{K} \sum_{i=1}^K z_i.
\]
Finally, we correct for the modulation bias to obtain our moment estimator,
\[
\widehat{\Gamma} = \frac{1}{1-\alpha} \bar{z}.
\]
The unbiasedness of this estimator follows from careful analysis of the cross-moment:
\begin{equation}
\begin{aligned}
\mathbb{E}[h_i y_i] &= \mathbb{E}\left[\left((1-\alpha)x_i + \lambda \cos(\omega\langle x_i, v\rangle + \phi_i)v + \eta_i\right)\left(\langle x_i, \beta^*\rangle + \varepsilon_i\right)\right] \\
&= \mathbb{E}\left[(1-\alpha)x_i\langle x_i, \beta^*\rangle\right] + \mathbb{E}\left[(1-\alpha)x_i\varepsilon_i\right] \\
&\quad + \mathbb{E}\left[\lambda \cos(\omega\langle x_i, v\rangle + \phi_i)v\langle x_i, \beta^*\rangle\right] + \mathbb{E}\left[\lambda \cos(\omega\langle x_i, v\rangle + \phi_i)v\varepsilon_i\right] \\
&\quad + \mathbb{E}\left[\eta_i\langle x_i, \beta^*\rangle\right] + \mathbb{E}\left[\eta_i\varepsilon_i\right].
\end{aligned}
\end{equation}Now upon expanding each term, we get the following.
$$\mathbb{E}\left[(1-\alpha)x_i\langle x_i, \beta^*\rangle\right] = (1-\alpha)\mathbb{E}[x_i x_i^\top]\beta^* = (1-\alpha)\Sigma_x \beta^*.$$
Since $\varepsilon_i$ is zero-mean and independent of $x_i$ we have,$$\mathbb{E}\left[(1-\alpha)x_i\varepsilon_i\right] = 0.$$
As $\phi_i$ is uniform on $[0,2\pi]$, making $\cos(\omega\langle x_i, v\rangle + \phi_i)$ zero-mean and independent of $x_i$ we have
$$\mathbb{E}\left[\lambda \cos(\omega\langle x_i, v\rangle + \phi_i)v\langle x_i, \beta^*\rangle\right] = 0.$$ 
Similarly, as $\eta_i$ is zero-mean and independent of $x_i$ we have,
 $\mathbb{E}\left[\lambda \cos(\omega\langle x_i, v\rangle + \phi_i)v\varepsilon_i\right] = 0$,
 $\mathbb{E}\left[\eta_i\langle x_i, \beta^*\rangle\right] = 0$,  
and $\mathbb{E}\left[\eta_i\varepsilon_i\right] = 0$. Therefore, $\mathbb{E}[h_i y_i] = (1-\alpha)\Sigma_x \beta^*$, which demonstrates that $\widehat{\Gamma}$ provides an unbiased estimate of the population moment as
\[
\mathbb{E}[\widehat{\Gamma}] = \Sigma_x \beta^*.
\]
The final parameter estimate thereby is obtained by solving the following using our covariance estimator and the cross-moment  $\widehat{\Gamma}$.
\[
\text{Solve } \hat{\beta} = \widehat{\Sigma}_x^{-1} \widehat{\Gamma} \text{ to recover } \beta^*.
\]
This one-step procedure avoids the privacy accumulation of iterative methods due to privacy properties under composition. That said, practically obtaining a good solution is also contingent on the matrix inverse and on dealing with the overall conditioning of this system.  The estimate of $\widehat{\Sigma}_x$ can be computed from corollary 3.3 where $h_i$ is used in place of $\tilde{g}_i$. As $h_i$ in non-iterative version and $\tilde{g}_i$ in iterative version are exactly the same.
\definecolor{algoheader}{RGB}{50,50,50}  

\section{Reconstruction error lower bound}
\label{sec:reconstruction}

The preceding sections establish that the protocol is useful in the sense that the server
recovers an unbiased gradient in expectation, while protecting the privacy. We now ask the complementary question
from an adversary's perspective: given what a client transmits, how accurately can
an attacker reconstruct that client's private feature vector? We derive a
Cram\'{e}r--Rao lower bound on the mean squared reconstruction error achievable
by any unbiased attacker, thereby quantifying the protocol's privacy protection
directly in terms of its parameters.

\paragraph{Setup.}
Under the single-vector protocol, client $i$ transmits
\begin{equation}
    \label{eq:client-transmission}
    (\tilde{g}_i,\, y_i),
    \qquad
    \tilde{g}_i = g_i + \xi_i,
    \quad
    g_i = (1-\alpha)x_i + \lambda\cos(\theta_i)\,v,
    \quad
    \theta_i = \omega\langle x_i, v\rangle + \varphi_i,
\end{equation}
where $\xi_i \sim \mathcal{N}(0,\sigma^2_{\mathrm{DP}}I_d)$,
$\varepsilon_i \sim \mathcal{N}(0,\sigma^2_y)$,
and $\varphi_i \sim \mathrm{Unif}[0,2\pi]$ are all mutually independent.
The attacker is assumed to know all public parameters:
$\beta$, $v$, $\alpha$, $\lambda$, $\omega$, $\sigma^2_{\mathrm{DP}}$,
$\beta^*$, and $\sigma^2_y$. Their goal is to construct an unbiased estimator
$\hat{x}$ of $x_i$ from the observation $(\tilde{g}_i, y_i)$.

\paragraph{Information decomposition.}
Because $\tilde{g}_i$ and $y_i$ are conditionally independent given $x_i$
(their randomness comes from independent sources: $\varphi_i$ and $\xi_i$ for
$\tilde{g}_i$, and $\varepsilon_i$ for $y_i$), the total Fisher information
about $x_i$ decomposes additively:
\begin{equation}
    \label{eq:fisher-decomp}
    \mathcal{I}_{(\tilde{g}_i,\, y_i)}(x_i)
    = \mathcal{I}_{\tilde{g}_i}(x_i) + \mathcal{I}_{y_i}(x_i).
\end{equation}
We compute each term via the following auxiliary lemmas.

\begin{lemma}[Fisher information for a Gaussian-perturbed map]
\label{lem:fisher-gaussian}
Let $\tilde{g} = f(x) + \xi$ where $\xi \sim \mathcal{N}(0,\sigma^2 I_d)$ and
$f:\mathbb{R}^d \to \mathbb{R}^d$ is differentiable with Jacobian $J_f(x)$.
Then the Fisher information matrix for $x$ given observation $\tilde{g}$ is
\[
    \mathcal{I}_{\tilde{g}}(x) = \frac{1}{\sigma^2}\,J_f(x)^\top J_f(x).
\]
\end{lemma}

\begin{proof}
Conditioned on $x$, we have $\tilde{g}\mid x \sim \mathcal{N}(f(x),\sigma^2 I_d)$
with log-likelihood
$\log p(\tilde{g}\mid x) = \mathrm{const} - \tfrac{1}{2\sigma^2}\|\tilde{g}-f(x)\|^2$.
The score is
$\nabla_x \log p(\tilde{g}\mid x) = \tfrac{1}{\sigma^2}J_f(x)^\top(\tilde{g}-f(x))$.
Since $\tilde{g}-f(x)=\xi \sim \mathcal{N}(0,\sigma^2 I_d)$,
\[
    \mathcal{I}_{\tilde{g}}(x)
    = \mathbb{E}\!\left[\nabla_x \log p\cdot(\nabla_x \log p)^\top\right]
    = \frac{1}{\sigma^4}\,J_f(x)^\top\,\mathbb{E}[\xi\xi^\top]\,J_f(x)
    = \frac{1}{\sigma^2}\,J_f(x)^\top J_f(x)
\]
\end{proof}

\begin{lemma}[Jacobian of the client map]
\label{lem:jacobian}
Fix the phase $\varphi_i$ and let $\theta_i = \omega\langle x_i,v\rangle + \varphi_i$.
The Jacobian of $x_i \mapsto g_i$ is
\[
    J_g(x_i) = (1-\alpha)I_d - \lambda\omega\sin(\theta_i)\,vv^\top.
\]
Consequently,
\[
    J_g(x_i)^\top J_g(x_i)
    = (1-\alpha)^2 I_d + c_1(x_i)\,vv^\top,
\]
where
\[
    c_1(x_i)
    := -2\lambda\omega(1-\alpha)\sin(\theta_i) + \lambda^2\omega^2\sin^2(\theta_i).
\]
\end{lemma}

\begin{proof}
Differentiating $g_i = (1-\alpha)x_i + \lambda\cos(\omega\langle x_i,v\rangle+\varphi_i)\,v$
with respect to $x_i$ gives
$J_g(x_i) = (1-\alpha)I_d - \lambda\omega\sin(\theta_i)\,vv^\top$.
Expanding the product and using $\|v\|=1$ (so $(vv^\top)^2 = vv^\top$),
\begin{align*}
    J_g^\top J_g
    &= \bigl[(1-\alpha)I_d - \lambda\omega\sin(\theta_i)\,vv^\top\bigr]^2 \\
    &= (1-\alpha)^2 I_d
       - 2\lambda\omega(1-\alpha)\sin(\theta_i)\,vv^\top
       + \lambda^2\omega^2\sin^2(\theta_i)\,vv^\top \\
    &= (1-\alpha)^2 I_d + c_1(x_i)\,vv^\top
\end{align*}
\end{proof}

\begin{theorem}[Conditional reconstruction lower bound]
\label{thm:crb-conditional}
Under the single-vector protocol, fix the client phase $\varphi_i$ and let
$\theta_i = \omega\langle x_i,v\rangle + \varphi_i$.
For any unbiased estimator $\hat{x}$ of $x_i$ constructed from $(\tilde{g}_i, y_i)$,
the per-dimension mean squared reconstruction error satisfies
\begin{equation}
    \label{eq:crb}
    \frac{1}{d}\,\mathbb{E}\!\left[\|\hat{x}-x_i\|^2\right]
    \;\geq\;
    \frac{1}{
        \dfrac{1}{\sigma^2_{\mathrm{DP}}}
        \!\left[(1-\alpha)^2 + \dfrac{c_1(x_i)}{d}\right]
        +\dfrac{\|\beta^*\|^2}{d\,\sigma^2_y}
    },
\end{equation}
where $c_1(x_i) = -2\lambda\omega(1-\alpha)\sin(\theta_i)
+ \lambda^2\omega^2\sin^2(\theta_i)$.
\end{theorem}

\begin{proof}
Fisher information from $\tilde{g}_i$ can be written upon conditioning on the fixed phase $\varphi_i$, as the map $x_i \mapsto g_i$ is
deterministic and differentiable based on lemmas~\ref{lem:fisher-gaussian}
and~\ref{lem:jacobian} as, 
\[
    \mathcal{I}_{\tilde{g}_i}(x_i \mid \varphi_i)
    = \frac{1}{\sigma^2_{\mathrm{DP}}}
      \bigl[(1-\alpha)^2 I_d + c_1(x_i)\,vv^\top\bigr].
\] Now, the Fisher information from $y_i$, based on
$y_i \mid x_i \sim \mathcal{N}(\langle\beta^*,x_i\rangle,\sigma^2_y)$,
the score is
$\nabla_{x_i}\log p(y_i\mid x_i)
= \tfrac{1}{\sigma^2_y}(y_i - \langle\beta^*,x_i\rangle)\beta^*$,
can be stated as, 
\[
    \mathcal{I}_{y_i}(x_i) = \frac{1}{\sigma^2_y}\,\beta^*{\beta^*}^\top.
\]
The total Fisher information therefore, is
\[
    \mathcal{I}(x_i)
    = \frac{1}{\sigma^2_{\mathrm{DP}}}
      \bigl[(1-\alpha)^2 I_d + c_1(x_i)\,vv^\top\bigr]
      + \frac{1}{\sigma^2_y}\,\beta^*{\beta^*}^\top.
\]
For any unbiased $\hat{x}$, the Cram\'{e}r--Rao inequality given by
$\mathrm{Cov}(\hat{x})\succeq \mathcal{I}(x_i)^{-1}$, gives
$\mathbb{E}[\|\hat{x}-x_i\|^2]
= \mathrm{tr}(\mathrm{Cov}(\hat{x}))
\geq \mathrm{tr}(\mathcal{I}(x_i)^{-1})$.
Applying the matrix-trace inequality $\mathrm{tr}(A^{-1})\geq d^2/\mathrm{tr}(A)$
for any positive definite $A\in\mathbb{R}^{d\times d}$ yields the stated inequality~\eqref{eq:crb} as,
\[
    \mathrm{tr}(\mathcal{I}(x_i))
    = \frac{(1-\alpha)^2 d + c_1(x_i)}{\sigma^2_{\mathrm{DP}}}
      + \frac{\|\beta^*\|^2}{\sigma^2_y}.
\]

\end{proof}

\begin{corollary}[Phase-averaged reconstruction lower bound]
\label{cor:crb-marginal}
Taking the expectation of the denominator of~\eqref{eq:crb} over the
uniform phase $\varphi_i \sim \mathrm{Unif}[0,2\pi]$, and using
$\mathbb{E}_{\varphi_i}[\sin\theta_i] = 0$ and
$\mathbb{E}_{\varphi_i}[\sin^2\theta_i] = \tfrac{1}{2}$,
we obtain
\[
    \mathbb{E}_{\varphi_i}[c_1(x_i)]
    = \frac{\lambda^2\omega^2}{2}
    \;\geq\; 0.
\]
Hence the following unconditional lower bound holds:
\begin{equation}
    \label{eq:crb-marginal}
    \frac{1}{d}\,\mathbb{E}\!\left[\|\hat{x}-x_i\|^2\right]
    \;\geq\;
    \frac{1}{
        \dfrac{(1-\alpha)^2}{\sigma^2_{\mathrm{DP}}}
        + \dfrac{\lambda^2\omega^2}{2d\,\sigma^2_{\mathrm{DP}}}
        + \dfrac{\|\beta^*\|^2}{d\,\sigma^2_y}
    }.
\end{equation}
This closed-form bound depends only on the protocol parameters and is
straightforward to evaluate and optimize. The modulation term
$\lambda^2\omega^2/(2d\sigma^2_{\mathrm{DP}})$ is always non-negative,
confirming that the cosine modulation never harms the attacker's lower
bound on average, and improves it whenever $\lambda\omega > 0$.
\end{corollary}

\subsection{Interpretation}\label{rem:phase-dependence}
The bound~\eqref{eq:crb} holds at each fixed realization of the phase
$\varphi_i$, through the dependence of $c_1(x_i)$ on
$\theta_i = \omega\langle x_i,v\rangle + \varphi_i$.
The effects can be concluded based on the denominator. Increasing $\sigma^2_{\mathrm{DP}}$ shrinks the first denominator term, directly increasing the lower bound and making reconstruction harder for any attacker. Similarly, when $\sin\theta_i > 0$ and the cross term $2\lambda\omega(1-\alpha)\sin\theta_i$ dominates, $c_1(x_i)$ is negative. This reduces the Fisher information from $\tilde{g}_i$ below the no-modulation level $(1-\alpha)^2 d/\sigma^2_{\mathrm{DP}}$, providing an additional layer of protection along the modulation direction $v$ beyond the Gaussian noise alone. The term $\|\beta^*\|^2/(d\sigma^2_y)$ reflects information about $x_i$ that leaks through the response $y_i$. Since $y_i$ is treated as non-sensitive in the protocol, this part is irreducible.
\label{sec:cramer-rao}



\section{Bias reduction by $O(\sqrt{m})$ with multiple orthonormal vectors}
We now describe a variant of this protocol that further reduces the bias, faster by using a set of multiple vectors $V_t = \{v_{t,j}\}_{j=1}^m$ instead of the single-vector version that we described in section \ref{singVec}. To summarize the results that follow, we show that the multi-vector version results in a lower Lipschitz constant, thereby resulting in requiring a smaller amount of noise for differential privacy than the single-vector version. We also show that bias obtained on the multi-vector version prior to the server-side correction reduces faster by an order of the square root of the data dimension, $\sqrt{m}$.
\label{def:setting}
Fix integers $d\ge 1$ and $m\in\{1,\dots,d\}$. Let $K\ge 1$ clients each hold a single pair $(x_i,y_i)$ with $x_i\in\R^d$ and $y_i\in\R$.
In each round $t$, the server holds $\beta_t\in\R^d$ and selects an orthonormal set
\[
V_t:=\{v_{t,1},\dots,v_{t,m}\}\subset \R^d,
\qquad v_{t,j}^\top v_{t,k}=\delta_{jk}.
\]

\begin{algorithm}[H]
\caption{Multi-Vector: Client Protocol }
\begin{algorithmic}[1]
\State \textcolor{white}{\colorbox{algoheader}{\makebox[\linewidth][l]{\bfseries Input:
Private data $(x_i,y_i)$, $(\alpha,\lambda,\omega,\sigma^2_{\text{DP}})$, $\beta$, $V=\{v_j\}_{j=1}^m$}}}
\State Sample $\phi_{i,j} \sim \mathrm{Unif}[0,2\pi]$ for each $j=1,\dots,m$
\State $g_i \leftarrow (1-\alpha)x_i + \dfrac{\lambda}{\sqrt{m}} \sum_{j=1}^m \cos(\omega \langle x_i, v_j\rangle + \phi_{i,j})v_j$
\State $\xi_i \sim \mathcal{N}(0,\sigma^2_{\text{DP}} I_d)$
\State $\tilde g_i \leftarrow g_i + \xi_i$
\State Send $(u_i,\tilde g_i)$ to the server
\end{algorithmic}
\end{algorithm}

\begin{algorithm}[H]
\caption{Multi-Vector: Server Protocol }
\begin{algorithmic}[1]
\State \textcolor{white}{\colorbox{algoheader}{\makebox[\linewidth][l]{\bfseries Input: $\{(\tilde g_i, y_i)\}_{i=1}^K$, $\beta$, $(\alpha,\lambda,\sigma^2_{\mathrm{DP}})$, $V=\{v_j\}_{j=1}^m$}}}
\State Compute projection matrix: $P_V \leftarrow \sum_{j=1}^m v_j v_j^{\top}$
\State $\widetilde{\Sigma}_g \leftarrow \dfrac{1}{K}\sum_{i=1}^K \tilde g_i \tilde g_i^{\top}$
\State $\widetilde{\Sigma}_x \leftarrow \dfrac{1}{(1-\alpha)^2}\Big(\widetilde{\Sigma}_g - \dfrac{\lambda^2}{2m} P_V - \sigma^2_{\mathrm{DP}} I\Big)$
\State $Z \leftarrow \dfrac{1}{1-\alpha} \cdot \dfrac{1}{K}\sum_{i=1}^K y_i \tilde g_i$
\State $G \leftarrow \widetilde{\Sigma}_x \beta - Z$
\State $\beta \leftarrow \beta - \eta\, G$
\State $\beta=\frac{\beta}{\max \left(1,\left\|\beta\right\|_2 / c\right)}$
\State Broadcast updated $\beta$ and new orthonormal set $V \perp \beta$ to clients
\end{algorithmic}
\end{algorithm}

Let $V_t\in\R^{d\times m}$ denote the matrix with columns $v_{t,1},\dots,v_{t,m}$, and define the rank-$m$ orthogonal projector
\[
P_{V_t}:=V_tV_t^\top=\sum_{j=1}^m v_{t,j}v_{t,j}^\top.
\]
Optionally, the server may impose $v_{t,j}\perp \beta_t$ for all $j$; if $\beta_t\neq 0$ this forces $m\le d-1$.

Privacy is defined with respect to the client feature $x_i$ using the adjacency relation $x\sim x'$ if $\norm{x-x'}_2\le 1$.
Throughout, $y_i$ is treated as non-sensitive/public (as in the draft's ``new'' single-vector protocol), and privacy is required only for $x_i$.

\begin{definition}[Multi-vector modulated client map and privatized transmission]\label{def:clientmap}
Fix parameters $\alpha\in(0,1)$, $\lambda>0$, and $\omega>0$. For a given orthonormal set $V=\{v_1,\dots,v_m\}$, a phase vector
$\phi=(\phi_1,\dots,\phi_m)\in[0,2\pi]^m$, and an input $x\in\R^d$, define the \emph{noise-free} modulated map
\begin{equation}\label{eq:gmap}
g_V(x;\phi)
:=(1-\alpha)x+\frac{\lambda}{\sqrt m}\sum_{j=1}^m \cos\!\bigl(\omega\,\ip{x}{v_j}+\phi_j\bigr)\,v_j.
\end{equation}
Given a noise scale $\sigma_{\mathrm{DP}}>0$, define the \emph{privatized} output
\begin{equation}\label{eq:tildegmap}
\tilde g_V(x;\phi):=g_V(x;\phi)+\xi,\qquad \xi\sim\mathcal{N}(0,\sigma_{\mathrm{DP}}^2 \Id_d),
\end{equation}
where $\xi$ is independent of $(x,\phi)$. In round $t$, client $i$ samples i.i.d.\ phases
\[
\phi_{i,j}^{(t)}\sim\mathrm{Unif}[0,2\pi],\qquad j=1,\dots,m,
\]
forms $g^{(t)}_i=g_{V_t}(x_i;\phi_i^{(t)})$ and $\tilde g^{(t)}_i=\tilde g_{V_t}(x_i;\phi_i^{(t)})$, and transmits \emph{only}
\begin{equation}\label{eq:transmitpair}
\bigl(\tilde g^{(t)}_i,\,y_i\bigr)
\end{equation}
to the server.
\end{definition}

\begin{theorem}[Global $\ell_2$ sensitivity of the multi-vector map]\label{thm:sensitivity}
Fix an orthonormal set $V=\{v_1,\dots,v_m\}$ and any phase vector $\phi\in[0,2\pi]^m$.
Let $g_V(\cdot;\phi)$ be defined in \eqref{eq:gmap}. Then for all $x,x'\in\R^d$,
\[
\norm{g_V(x;\phi)-g_V(x';\phi)}_2 \le L\,\norm{x-x'}_2,
\]
where one may take
\[
L:=|1-\alpha|+\lambda\omega\sqrt{\frac{s}{m}},
\qquad
s:=\sup_{h\neq 0}\frac{\norm{P_V h}_2^2}{\norm{h}_2^2}.
\]
Moreover, for orthonormal $V$, $P_V$ is an orthogonal projector and $\norm{P_V}_{\mathrm{op}}=1$, hence $s=1$ and
\[
L=|1-\alpha|+\frac{\lambda\omega}{\sqrt m}.
\]
In particular, under the adjacency relation $x\sim x'$ iff $\norm{x-x'}_2\le 1$, the $\ell_2$-sensitivity is bounded by $\Delta_2\le L$.
\end{theorem}

\begin{proof}
Fix $x,x'\in\R^d$ and let $h:=x-x'$. For $j\in\{1,\dots,m\}$ define
\[
\theta_j:=\omega\,\ip{x}{v_j}+\phi_j,\qquad \theta'_j:=\omega\,\ip{x'}{v_j}+\phi_j.
\]
Then
\[
g_V(x;\phi)-g_V(x';\phi)
=(1-\alpha)h+\frac{\lambda}{\sqrt m}\sum_{j=1}^m\bigl(\cos\theta_j-\cos\theta'_j\bigr)\,v_j.
\]
By the triangle inequality,
\[
\norm{g_V(x;\phi)-g_V(x';\phi)}_2
\le |1-\alpha|\norm{h}_2+\frac{\lambda}{\sqrt m}\,\norm{\sum_{j=1}^m(\cos\theta_j-\cos\theta'_j)v_j}_2.
\]
Since $\cos(\cdot)$ is $1$-Lipschitz on $\R$, $|\cos u-\cos w|\le |u-w|$ for all $u,w\in\R$. Therefore,
\[
|\cos\theta_j-\cos\theta'_j|
\le |\theta_j-\theta'_j|
=\omega\,|\ip{h}{v_j}|.
\]
Set $c_j:=\cos\theta_j-\cos\theta'_j$. Then $|c_j|\le \omega|\ip{h}{v_j}|$ and, using orthonormality,
\[
\norm{\sum_{j=1}^m c_j v_j}_2^2=\sum_{j=1}^m c_j^2
\le \omega^2\sum_{j=1}^m \ip{h}{v_j}^2
=\omega^2\norm{P_V h}_2^2.
\]
Taking square roots yields
\[
\norm{\sum_{j=1}^m c_j v_j}_2\le \omega\norm{P_V h}_2\le \omega\sqrt{s}\,\norm{h}_2,
\]
where $s=\sup_{h\neq 0}\norm{P_V h}_2^2/\norm{h}_2^2$. Substituting back gives
\[
\norm{g_V(x;\phi)-g_V(x';\phi)}_2
\le \Bigl(|1-\alpha|+\lambda\omega\sqrt{\tfrac{s}{m}}\Bigr)\norm{h}_2,
\]
which proves the first claim with $L:=|1-\alpha|+\lambda\omega\sqrt{s/m}$.

If $V$ is orthonormal then $P_V$ is an orthogonal projector, hence $\norm{P_V}_{\mathrm{op}}=1$ and $s=1$.
The sensitivity statement for adjacency $\norm{x-x'}_2\le 1$ is immediate.
\end{proof}

\begin{corollary}[Per-round local differential privacy via the Gaussian mechanism]\label{cor:ldp}
Let $\Delta_2\le L$ be the sensitivity bound from Theorem~\ref{thm:sensitivity}. Fix $\varepsilon>0$ and $\delta\in(0,1)$ and set
\begin{equation}\label{eq:sigmadp}
\sigma_{\mathrm{DP}}:=\frac{L\sqrt{2\ln(1.25/\delta)}}{\varepsilon}.
\end{equation}
Then, for each round and each client, the randomized mapping $x\mapsto \tilde g_V(x;\Phi)$ of \eqref{eq:tildegmap}, where $\Phi$ is the client-sampled random phase vector and $\xi$ is Gaussian noise with variance $\sigma_{\mathrm{DP}}^2$, is $(\varepsilon,\delta)$-locally differentially private with respect to $x$.
\end{corollary}

\begin{proof}
Condition on an arbitrary realization of the phases $\Phi=\phi$. The map $x\mapsto g_V(x;\phi)$ is deterministic and has $\ell_2$-sensitivity at most $L$ by Theorem~\ref{thm:sensitivity}.
By the standard Gaussian mechanism for vector-valued queries with $\ell_2$-sensitivity $L$, adding $\xi\sim\mathcal{N}(0,\sigma_{\mathrm{DP}}^2\Id_d)$ with $\sigma_{\mathrm{DP}}$ as in \eqref{eq:sigmadp} yields $(\varepsilon,\delta)$-DP conditional on $\Phi=\phi$.
Since $\Phi$ is sampled independently of $x$ and DP is preserved under additional randomness independent of the private input, the unconditional mechanism $x\mapsto \tilde g_V(x;\Phi)$ is also $(\varepsilon,\delta)$-locally differentially private.
\end{proof}

\begin{lemma}[Second-moment identity for $\tilde g_i$]\label{lem:secondmoment}
Fix any client $i$ and treat $(x_i,y_i)$ as fixed.
Let $\tilde g_i$ be generated according to Definitions~\ref{def:clientmap} using orthonormal $V=\{v_j\}_{j=1}^m$,
i.i.d.\ phases $\phi_{i,j}\sim\mathrm{Unif}[0,2\pi]$, and independent Gaussian noise $\xi_i\sim\mathcal{N}(0,\sigma_{\mathrm{DP}}^2\Id_d)$.
Then
\begin{equation}\label{eq:secondmomentidentity}
\E\bigl[\tilde g_i\tilde g_i^\top \,\big|\, x_i\bigr]
=(1-\alpha)^2 x_ix_i^\top+\frac{\lambda^2}{2m}P_V+\sigma_{\mathrm{DP}}^2\Id_d,
\end{equation}
where the expectation is over the internal randomness $(\phi_i,\xi_i)$.
\end{lemma}

\begin{proof}
Write $\tilde g_i=g_i+\xi_i$ with $\E[\xi_i]=0$ and $\E[\xi_i\xi_i^\top]=\sigma_{\mathrm{DP}}^2\Id_d$. Then
\[
\E[\tilde g_i\tilde g_i^\top\mid x_i]
=\E[g_ig_i^\top\mid x_i]+\E[\xi_i\xi_i^\top]+\E[g_i\xi_i^\top\mid x_i]+\E[\xi_ig_i^\top\mid x_i].
\]
The cross terms vanish because $\xi_i$ is independent of $g_i$ and mean-zero, hence
$\E[g_i\xi_i^\top\mid x_i]=\E[g_i\mid x_i]\E[\xi_i]^\top=0$ and similarly $\E[\xi_ig_i^\top\mid x_i]=0$.
Therefore,
\[
\E[\tilde g_i\tilde g_i^\top\mid x_i]=\E[g_ig_i^\top\mid x_i]+\sigma_{\mathrm{DP}}^2\Id_d.
\]

Next decompose $g_i=A_i+B_i$ where
\[
A_i:=(1-\alpha)x_i,
\qquad
B_i:=\frac{\lambda}{\sqrt m}\sum_{j=1}^m \cos(\theta_{i,j})\,v_j,
\qquad
\theta_{i,j}:=\omega\,\ip{x_i}{v_j}+\phi_{i,j}.
\]
Then
\[
\E[g_ig_i^\top\mid x_i]=A_iA_i^\top+\E[B_iB_i^\top\mid x_i]+\E[A_iB_i^\top\mid x_i]+\E[B_iA_i^\top\mid x_i].
\]
Since $\phi_{i,j}$ is uniform on $[0,2\pi]$ and independent of $x_i$, $\E[\cos(\theta_{i,j})\mid x_i]=0$, so $\E[B_i\mid x_i]=0$.
Hence $\E[A_iB_i^\top\mid x_i]=A_i\E[B_i^\top\mid x_i]=0$ and similarly $\E[B_iA_i^\top\mid x_i]=0$.
Thus,
\[
\E[g_ig_i^\top\mid x_i]=A_iA_i^\top+\E[B_iB_i^\top\mid x_i]=(1-\alpha)^2x_ix_i^\top+\E[B_iB_i^\top\mid x_i].
\]

It remains to compute $\E[B_iB_i^\top\mid x_i]$. Expanding,
\[
B_iB_i^\top
=\frac{\lambda^2}{m}\sum_{j=1}^m\sum_{k=1}^m
\cos(\theta_{i,j})\cos(\theta_{i,k})\,v_jv_k^\top.
\]
For $j\neq k$, independence of $\phi_{i,j}$ and $\phi_{i,k}$ implies
$\E[\cos(\theta_{i,j})\cos(\theta_{i,k})\mid x_i]=\E[\cos(\theta_{i,j})\mid x_i]\E[\cos(\theta_{i,k})\mid x_i]=0$.
For $j=k$, uniformity of $\phi_{i,j}$ implies $\theta_{i,j}\ (\mathrm{mod}\ 2\pi)$ is uniform on $[0,2\pi]$, so
\[
\E[\cos^2(\theta_{i,j})\mid x_i]=\frac{1}{2\pi}\int_0^{2\pi}\cos^2(u)\,du=\frac12.
\]
Therefore,
\[
\E[B_iB_i^\top\mid x_i]
=\frac{\lambda^2}{m}\sum_{j=1}^m \frac12\,v_jv_j^\top
=\frac{\lambda^2}{2m}\sum_{j=1}^m v_jv_j^\top
=\frac{\lambda^2}{2m}P_V.
\]
Substituting into the earlier identity yields \eqref{eq:secondmomentidentity}.
\end{proof}

\begin{corollary}[Reduction to the single-vector  protocol]\label{cor:singlevector}
If $m=1$ and $V=\{v\}$, then $P_V=vv^\top$, the client map \eqref{eq:gmap} becomes
$g(x)=(1-\alpha)x+\lambda\cos(\omega\ip{x}{v}+\phi)v$,
the sensitivity bound reduces to $L=|1-\alpha|+\lambda\omega$,
and with the server side processing reduces to
\[
\widehat{\Sig}_x=\frac{1}{(1-\alpha)^2}\Bigl(\Sig_{\tilde g}-\frac{\lambda^2}{2}vv^\top-\sigma_{\mathrm{DP}}^2\Id_d\Bigr),
\]
which is exactly the structure of the draft's single-vector ``new'' protocol.
\end{corollary}

\begin{proof}
All identities follow by substituting $m=1$ into Definitions~\ref{def:clientmap} and observing that $P_V=\sum_{j=1}^1 v_jv_j^\top=vv^\top$.
\end{proof}

\section{Convergence} \cite{nemirovski2009robust}
Having established unbiasedness of the gradient estimator, we next quantify its stochastic variability. This section derives the variance decomposition of the estimator, which is then used as a key input needed for understanding the convergence guarantee of our updates.
To start with, the gradient estimator is given by
\[
G = \widehat{\Sigma}_{x}\beta - Z,\qquad 
Z = \frac{1}{1 - \alpha}\cdot \frac{1}{K}\sum_{i = 1}^{K}y_{i}\tilde{g}_{i},
\]
which satisfies \(\mathbb{E}[G] = \nabla L(\beta)\). We now derive the total variance of this estimator. Define the residual vector \(r\coloneqq X\beta - Y\in \mathbb{R}^{K}\), so \(r_{i} = \langle x_{i},\beta \rangle - y_{i}\).  
Let \(C\in \mathbb{R}^{K}\) have entries \(C_{i}\coloneqq \cos (\omega \langle x_{i},v\rangle +\phi_{i})\),  
let \(\Xi \in \mathbb{R}^{K\times d}\) be the noise matrix whose rows are the \(\xi_{i}^{\top}\), and set \(q\coloneqq \Xi \beta \in \mathbb{R}^{K}\), so \(q_{i} = \xi_{i}^{\top}\beta\).  
All phases \(\phi_{i}\) and all rows of \(\Xi\) are mutually independent.

The final formula involves only three scalar data‑dependent constants:
\[
S_{r}^{2}\coloneqq \frac{1}{K}\| r\|^{2},\qquad 
\beta^{\top}\Sigma_{r x}\coloneqq \frac{1}{K}\beta^{\top}X^{\top}r,\qquad 
\mathrm{tr}(\Sigma_{x})\coloneqq \frac{1}{K}\| X\|_{F}^{2} 
\]

\begin{theorem}\label{thm:var}
Under the assumption \(v\perp \beta\) with fixed data \((x_{1},y_{1}),\ldots ,(x_{K},y_{K})\)

\begin{multline}
\sigma_{g}^{2}(\beta) = \mathbb{E}\big[\| G - \nabla L(\beta)\|^{2}\big] \\
= \frac{ S_{r}^{2}\bigl(\frac{\lambda^{2}}{2} + d\sigma_{\mathrm{DP}}^{2}\bigr) + 2\sigma_{\mathrm{DP}}^{2}\beta^{\top}\Sigma_{rx} + \sigma_{\mathrm{DP}}^{2}\|\beta\|^{2}\operatorname{tr}(\Sigma_{x}) }{K(1-\alpha)^{2}} \\
+ \frac{ \|\beta\|^{2}\bigl(\frac{\lambda^{2}\sigma_{\mathrm{DP}}^{2}}{2} + (d+1)\sigma_{\mathrm{DP}}^{4}\bigr) }{K(1-\alpha)^{4}}.
\end{multline}
\end{theorem}

\begin{proof}
By the centering identity of \ref{lem:centering},
\[
G - \nabla L(\beta) = \frac{T_1 + T_2}{K(1 - \alpha)} +\frac{T_3}{K(1 - \alpha)^2},
\]
where \(T_1,T_2,T_3\) are the mean‑zero random vectors defined in (A.1)–(A.3). Squaring and taking expectations yields
\[
\sigma_g^2 (\beta) = \frac{\mathbb{E}[\|T_1 + T_2\|^2]}{K^2(1 - \alpha)^2} + \frac{\mathbb{E}[\|T_3\|^2]}{K^2(1 - \alpha)^4},
\]
because the cross term \(\mathbb{E}[(T_1+T_2)^\top T_3] = 0\) (Theorem C.2).  
Lemma C.1 gives
\[
\mathbb{E}[\| T_1 + T_2\| ^2 ] = \| r\| ^2 \left(\frac{\lambda^2}{2} +d\sigma_{\mathrm{DP}}^2\right) + 2K\sigma_{\mathrm{DP}}^2 \beta^\top \Sigma_{rx} + K\sigma_{\mathrm{DP}}^2 \| \beta \| ^2 \mathrm{tr}(\Sigma_{x}). 
\]
Lemma D.2 gives
\[
\mathbb{E}[\| T_3\| ^2 ] = K\| \beta \| ^2 \left(\frac{\lambda^2\sigma_{\mathrm{DP}}^2}{2} +(d + 1)\sigma_{\mathrm{DP}}^4\right). 
\]
Substituting above two equations into the expression for \(\sigma_g^2\) and using \(\|r\|^2 = K S_r^2\) yields the stated formula.
\end{proof}

\section{Uniform Variance Bound and Convergence}

\begin{corollary}\label{cor:uniform}
Suppose \(\| \beta_t\| \leq B\), \(\| x_i\| \leq R\), and \(|y_i| \leq M\) for all \(i\) and \(t\). Then \(\sigma_g^2 (\beta_t) \leq \bar{\sigma}_g^2\) where
\[
\bar{\sigma}_g^2 = \frac{ (RB+M)^2\bigl(\frac{\lambda^2}{2}+d\sigma_{\mathrm{DP}}^2\bigr) + 2\sigma_{\mathrm{DP}}^2 BR(RB+M) + \sigma_{\mathrm{DP}}^2 B^2 R^2 }{K(1-\alpha)^2}
+ \frac{ B^2\bigl(\frac{\lambda^2\sigma_{\mathrm{DP}}^2}{2} + (d+1)\sigma_{\mathrm{DP}}^4\bigr) }{K(1-\alpha)^4}.
\]
\end{corollary}

\begin{proof}
Since \(|r_i| = |\langle x_i, \beta \rangle - y_i| \leq \| x_i\| \| \beta \| + |y_i| \leq RB + M\), we have \(S_r^2 \leq (RB + M)^2\). By Cauchy‑Schwarz, \(|\beta^\top \Sigma_{rx}| \leq \| \beta \| \| \Sigma_{rx}\| \leq B \cdot R(RB + M)\), and \(\operatorname{tr}(\Sigma_x) \leq R^2\). Substituting these bounds into Theorem~\ref{thm:var} gives the claimed bound.
\end{proof}

\begin{theorem}\label{thm:conv}
Under the assumptions of Theorem~\ref{thm:var} and Corollary~\ref{cor:uniform}, with step size \(\eta = 1 / L_{\mathrm{DP}}\) where \(L_{\mathrm{DP}} = \lambda_{\max}(\Sigma_x)\), and $  \bar{\beta}_T=  \sum_t \beta_t/T$, we have
\[
\mathbb{E}[F_e(\bar{\beta}_T)] - F_e(\beta^*)\leq \frac{L_{\mathrm{DP}}\| \beta_0 - \beta^*\|^2}{2T} +\frac{\bar{\sigma}_g^2}{2L_{\mathrm{DP}}K}.
\]
\end{theorem}

\begin{proof}
This follows from standard theory of projected stochastic gradient descent on the set ${\beta: ||\beta|| <= c}$ \cite{nemirovski2009robust} \(\bar{\sigma}_g^2\) with smoothness constant \(L_{\mathrm{DP}}\). The first term decays as \(O(1/T)\) while the second term is the irreducible variance floor set by the privacy noise and protocol parameters.
\end{proof}

\section{Parameters} 
With privacy, unbiasedness, and convergence in place, we can now interpret the formulas at a design level, to help guide the user on choosing their values. This roughly translates the theory into practical guidance on how $\alpha, \lambda, \omega, m, d$, and $K$ shape the privacy-utility trade-off.
Since, $\sigma_{\mathrm{DP}}^2\propto L^2/\varepsilon^2$ with $L=|1-\alpha|+\lambda\omega$,
the dominant cost terms scale as $d\sigma_{\mathrm{DP}}^2/(1-\alpha)^2$
and $(d+1)\sigma_{\mathrm{DP}}^4/[K(1-\alpha)^4]$.
Reducing $L$ directly reduces $\sigma_{\mathrm{DP}}^2$ and both terms simultaneously.

With $m$ orthonormal modulation vectors the Lipschitz constant becomes
$L=|1-\alpha|+\lambda\omega/\sqrt{m}$, so $\sigma_{\mathrm{DP}}^2$ decreases with $m$.
Every $\sigma_{\mathrm{DP}}^2$-dependent term therefore benefits from
using more vectors. Increasing $\alpha$ lowers $L$ (and hence $\sigma_{\mathrm{DP}}^2$), but the denominators
$(1-\alpha)^2$ and $(1-\alpha)^4$ simultaneously amplify the variance.
The optimal $\alpha$ balances these opposing effects. The dimension appears linearly in both the $d\sigma_{\mathrm{DP}}^2 S_r^2$ and
$(d+1)\sigma_{\mathrm{DP}}^4$ terms, confirming a linear curse of dimensionality.

\section*{Experiments}

We evaluate the proposed modulated learner on five real regression tasks in the one-sample-per-client regime. Each training example is treated as a separate client. We compare the iterative modulated estimator from Algorithms~3.1--3.2, the one-shot estimator from Section~5 with ridge stabilization, and a tuned DP-SGD FedAvg baseline. The purpose of this appendix is to compare whether the privacy-utility trade-offs achieved by these methods in practice. 

\paragraph{Methods.} \emph{Modulated iterative} is the multi-round server-side update induced by the client mechanism of Algorithm~3.1 and the server correction of Algorithm~3.2. At round $t$, each client sends $(\tilde g_i, y_i)$; the server forms the debiased covariance estimate $\widehat\Sigma_x$ and corrected first moment $Z$, and updates
\[
\beta \leftarrow \beta - \eta\bigl(\widehat\Sigma_x\beta - Z\bigr).
\]
 \emph{Modulated one-shot ridge} uses the Section~5 moment estimator together with a ridge term (like in ridge regression) when solving the linear model. \emph{DP-SGD FedAvg} performs one clipped and Gaussian-perturbed local gradient step per client per round, followed by server averaging. All methods are operated  over the common privacy values $\epsilon\in\{0.5,0.75,1.0,1.25,\ldots,10.0\}$ while using a public validation split. For modulated iterative we tune the scalar factor $c$ in the adaptive rule $\eta_t = c / \lVert \widehat\Sigma_{x,t}\rVert_{\mathrm{op}}$. For the one-shot estimator we tune the ridge parameter $\gamma$. For DP-SGD we tune the clipping norm $C$ jointly with the learning rate. Each plotted curve therefore corresponds to one dataset-specific hyperparameter configuration held fixed across the entire privacy sweep.
\begin{figure}[!ht]
\centering
\begin{minipage}[t]{0.48\linewidth}
    \centering
    \includegraphics[width=\linewidth]{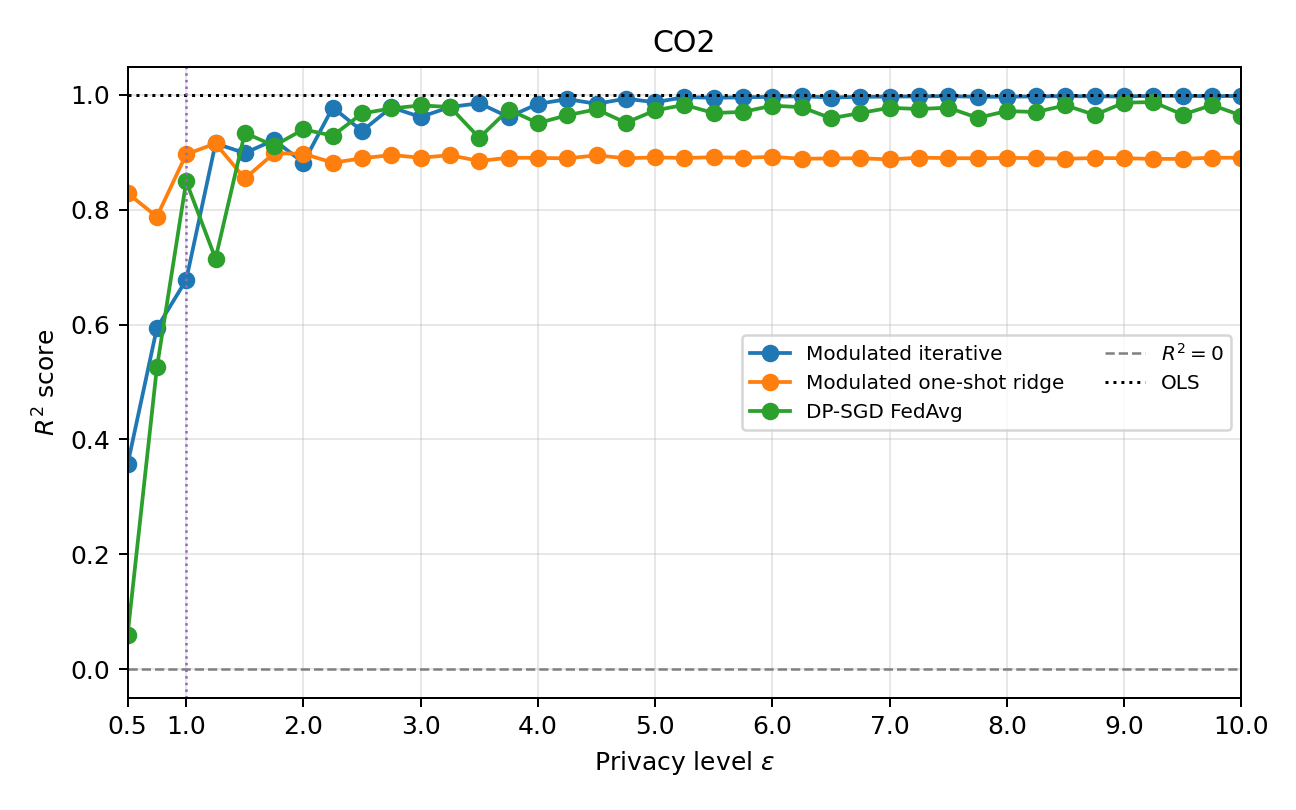}\\[-0.35em]
    {\small\textbf{(a)} CO$_2$}
\end{minipage}\hfill
\begin{minipage}[t]{0.48\linewidth}
    \centering
    \includegraphics[width=\linewidth]{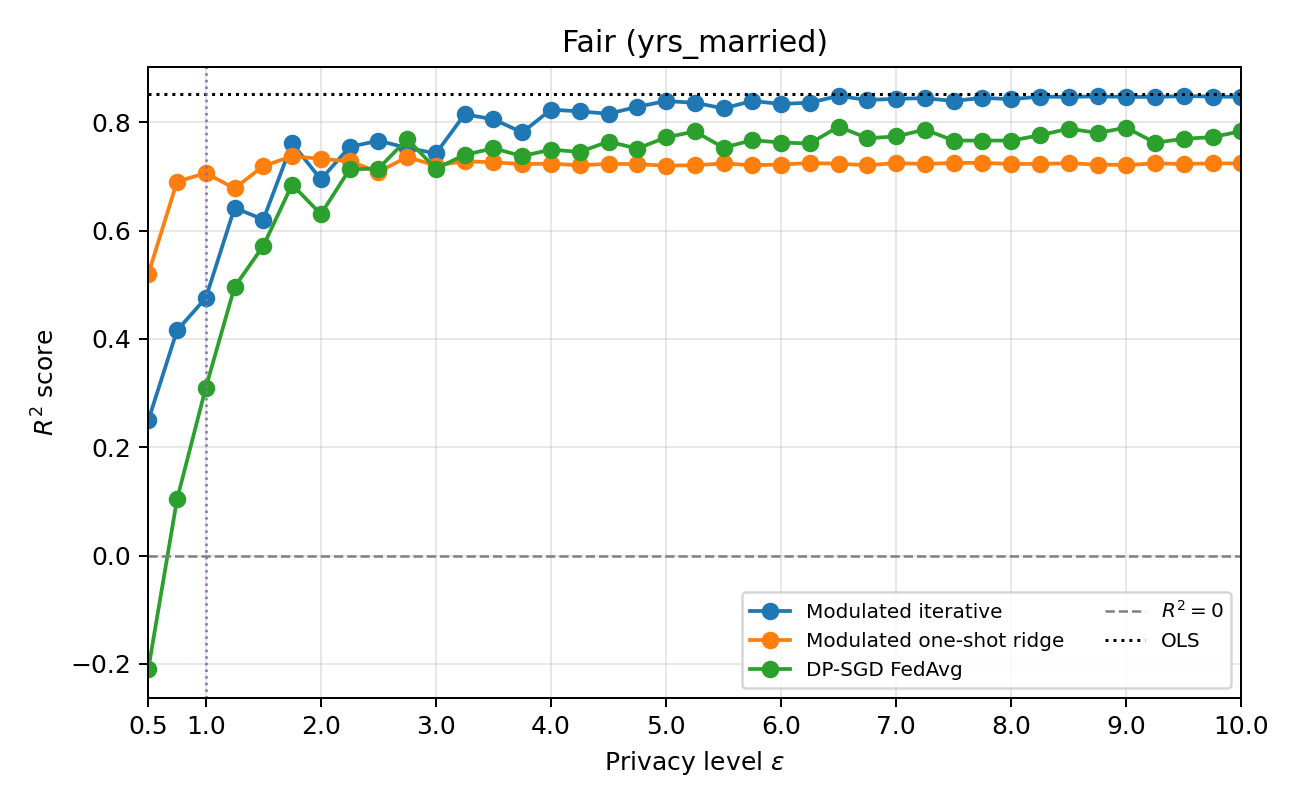}\\[-0.35em]
    {\small\textbf{(b)} Fair (years married)}
\end{minipage}

\begin{minipage}[t]{0.48\linewidth}
    \centering
    \includegraphics[width=\linewidth]{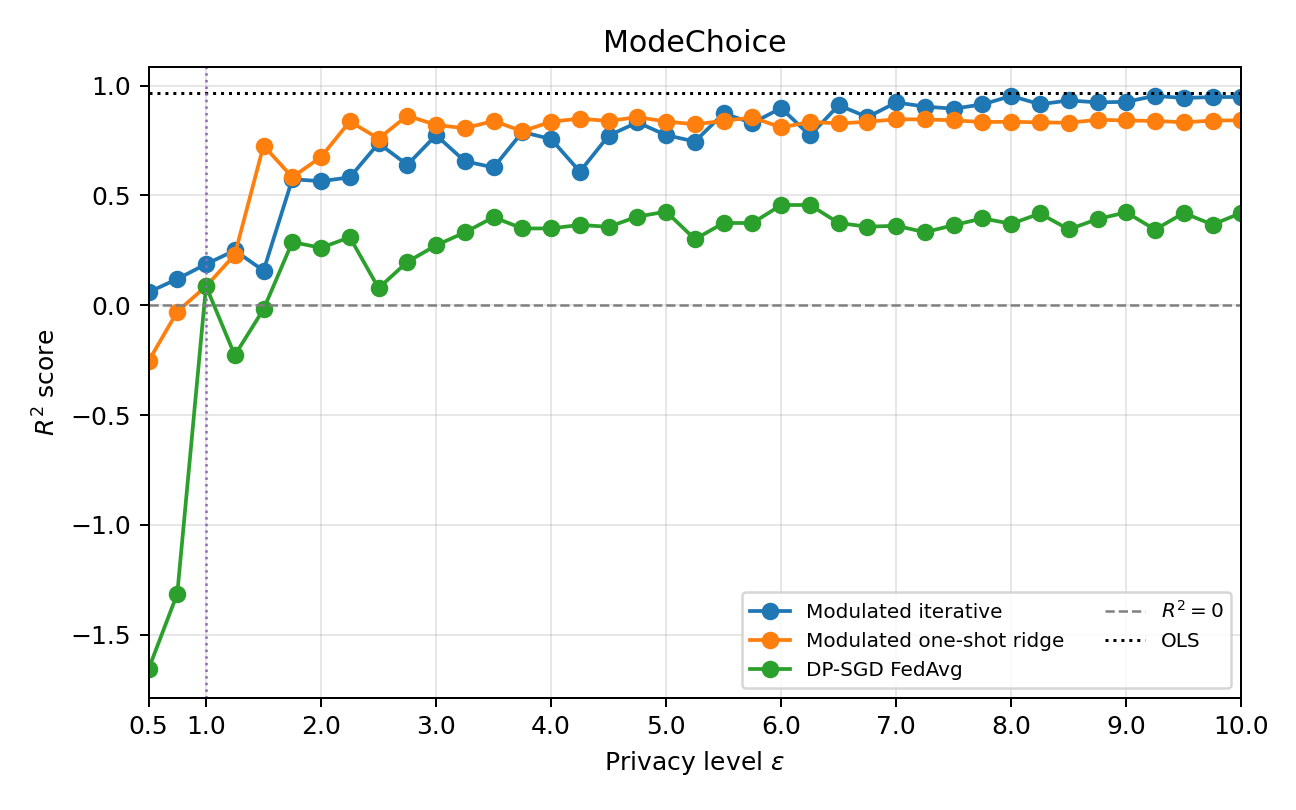}\\[-0.35em]
    {\small\textbf{(c)} ModeChoice}
\end{minipage}\hfill
\begin{minipage}[t]{0.48\linewidth}
    \centering
    \includegraphics[width=\linewidth]{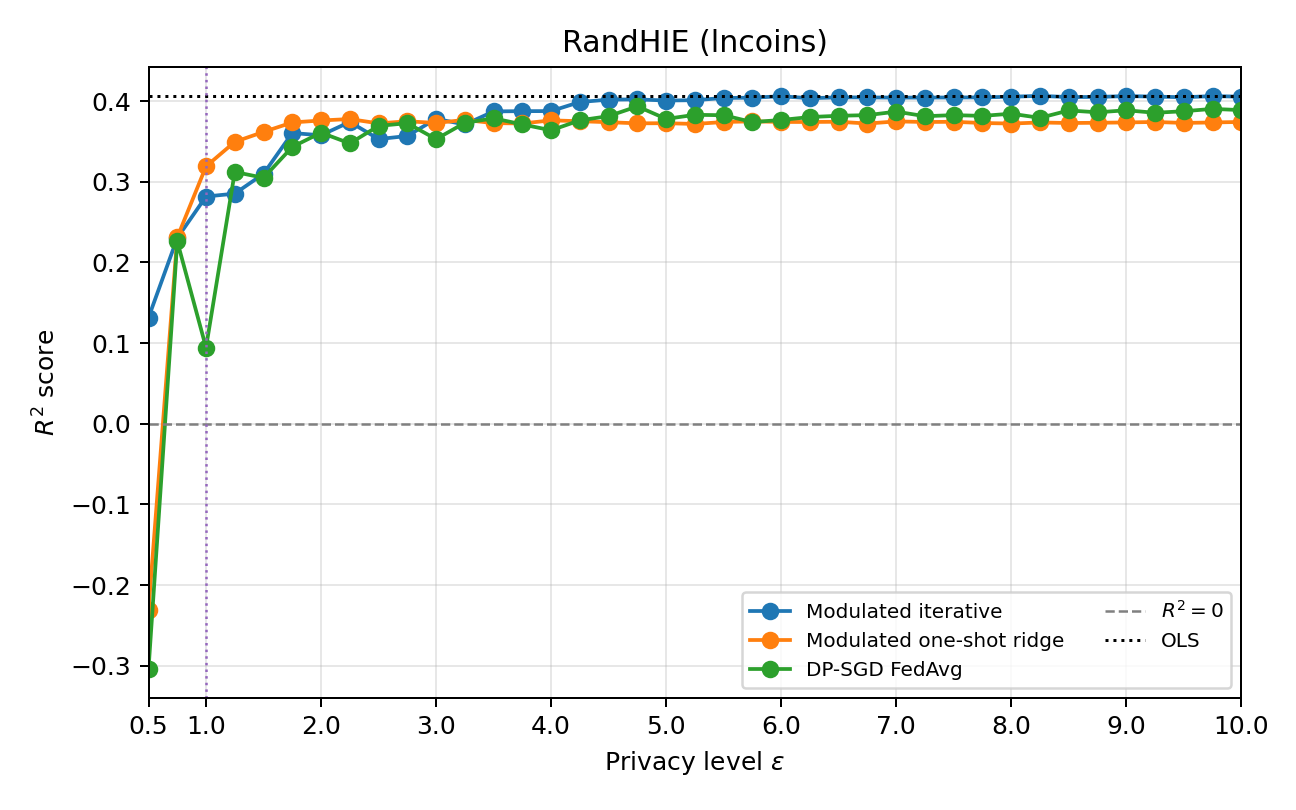}\\[-0.35em]
    {\small\textbf{(d)} RandHIE (\texttt{lncoins})}
\end{minipage}

\begin{minipage}[t]{0.60\linewidth}
    \centering
    \includegraphics[width=\linewidth]{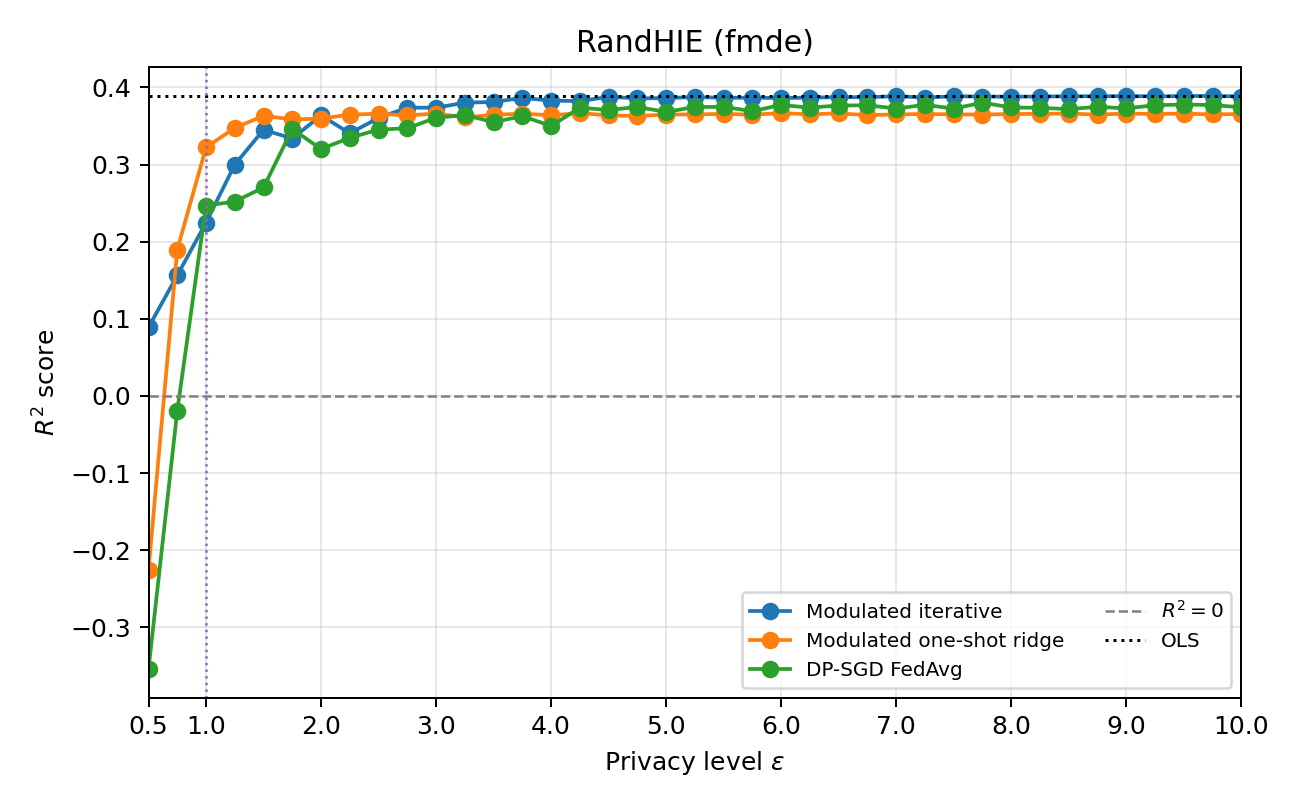}\\[-0.35em]
    {\small\textbf{(e)} RandHIE (\texttt{fmde})}
\end{minipage}
\caption{Jointly tuned test-set $R^2$ curves on five representative real tasks. Each panel compares modulated iterative, modulated one-shot ridge, and tuned DP-SGD FedAvg over the same privacy grid. The dashed horizontal line is the non-private OLS reference. The common pattern is that the one-shot modulated estimator is strongest at the smallest $\epsilon$ values, while the modulated iterative estimator improves more rapidly once the privacy budget is large enough to support repeated rounds.}
\label{fig:joint-realdata-r2-panels}
\end{figure}

\paragraph{Privacy accounting.} We follow the privacy convention used by the algorithm: privacy is imposed on the client feature vector $x_i$, while $y_i$ is treated as public in the server correction term. Smaller $\epsilon$ means stronger privacy and larger $\epsilon$ means weaker privacy. For the modulated methods, the client map is calibrated by the Gaussian mechanism using the sensitivity bound from Theorem~3.1. Since every client participates in every round, privacy composes sequentially across rounds under zCDP: if the per-round budget is $\rho_t$, then $\rho_{1:T}=\sum_{t=1}^T \rho_t$ and the corresponding $(\epsilon,\delta)$ guarantee is recovered through $\epsilon = \rho_{1:T} + 2\sqrt{\rho_{1:T}\log(1/\delta)}$. The iterative method uses $T=10$ rounds; the one-shot estimator uses a single private release. DP-SGD is evaluated with the same full-participation zCDP accountant over $T=10$ rounds, with per-round sensitivity $2C$ for the clipped gradient. There is no subsampling amplification because every client participates in every round. That said, it could still be applied to all the methods to obtain a similar conclusion on their performance. 

\paragraph{Datasets and preprocessing.} Each task is split once into train, validation, and test subsets. Features are standardized on the training split. The response is standardized for fitting, and $R^2$ is reported on the standardized test target. The dashed horizontal reference in every panel is the non-private OLS fit computed on the same split.

\begin{table}
\centering
\caption{Real datasets used in the appended comparison. Each training example is treated as one client.}
\label{tab:joint-realdata-datasets}
\begin{tabular}{lrrr}
\toprule
Dataset & $n$ & $d$ & OLS $R^2$ \\
\midrule
CO2 & 2284 & 7 & 0.999 \\
Fair (yrs\_married) & 6366 & 8 & 0.853 \\
ModeChoice & 840 & 7 & 0.967 \\
RandHIE (lncoins) & 20190 & 9 & 0.406 \\
RandHIE (fmde) & 20190 & 9 & 0.389 \\
\bottomrule
\end{tabular}
\end{table}
\FloatBarrier\paragraph{Results.} The five tasks show the same qualitative pattern. At the strongest privacy levels, the one-shot modulated estimator is often strongest because it uses a single private release and therefore avoids inter-round composition. As the privacy budget increases, the iterative estimator improves steadily and eventually overtakes the one-shot alternative on most tasks. The gain is largest on ModeChoice, where multi-round correction produces a clear advantage once the privacy budget is large enough to support repeated optimization. The two RandHIE targets have weaker linear signal, so the separation between methods is smaller, but the same transition remains visible. DP-SGD is a meaningful baseline throughout, especially on the easiest tasks, but in this regime it pays the full cost of ten rounds of composition while using only one clipped example per client per round.\par When the needed privacy is high (low $\epsilon$), the one-shot modulated estimator fared as the better choice because it extracts as much utility as possible from a single private release. As privacy weakens, the iterative modulated estimator becomes preferable because it can convert additional privacy budget into useful optimization progress over multiple rounds. 
\begin{table}
\centering
\caption{Jointly tuned hyperparameters over the considered range of $\epsilon$ values. The iterative column reports the selected step factor $c$, the one-shot column reports the selected ridge parameter $\gamma$, and the DP-SGD columns report the selected clipping norm $C$ and learning rate.}
\label{tab:joint-realdata-tuning}
\begin{tabular}{lcccc}
\toprule
Dataset & Iterative $c$ & One-shot $\gamma$ & DP-SGD $C$ & DP-SGD step \\
\midrule
CO2 & 0.80 & 2.00 & 2.249 & 0.05 \\
Fair (yrs\_married) & 0.80 & 1.00 & 2.059 & 0.10 \\
ModeChoice & 1.00 & 0.50 & 1.816 & 0.05 \\
RandHIE (lncoins) & 0.80 & 0.50 & 2.270 & 0.10 \\
RandHIE (fmde) & 0.50 & 0.50 & 2.256 & 0.10 \\
\bottomrule
\end{tabular}
\end{table}

\appendix

\section{Centering Identity}\label{app:centering}
This appendix proves the centering identity used to decompose the gradient noise in the variance analysis. The purpose of the result is to separate the random contributions into blocks that can be analyzed independently.

\begin{lemma}[Centering identity]\label{lem:centering}
Define the three $\R^d$-valued random vectors
\begin{align}
  T_1 &:= \lambda(C^\top r)\,v + \Xi^\top r, \label{eq:T1}\\
  T_2 &:= X^\top q, \label{eq:T2}\\
  T_3 &:= \lambda(C^\top q)\,v + \Xi^\top q - K\sigma_{\mathrm{DP}}^2\beta. \label{eq:T3}
\end{align}
Each of $T_1$, $T_2$, $T_3$ has mean zero, and
\[
  G - \nabla L(\beta)
  = \frac{T_1+T_2}{K(1-\alpha)} + \frac{T_3}{K(1-\alpha)^2}.
\]
\end{lemma}

\begin{proof}
Recall from Algorithm~3.2 of the main paper that
$G = \widehat\Sigma_x\beta - Z$,
where
$\widehat\Sigma_x = \tfrac{1}{(1-\alpha)^2}(\widetilde\Sigma_g - \tfrac{\lambda^2}{2}vv^\top - \sigma_{\mathrm{DP}}^2 I)$,
$\widetilde\Sigma_g = \tfrac{1}{K}\sum_i \tilde g_i\tilde g_i^\top$,
and $Z = \tfrac{1}{(1-\alpha)K}\sum_i y_i\tilde g_i$.

Since $\tilde g_i = (1-\alpha)x_i + \lambda C_i v + \xi_i$ and $v\perp\beta$, the inner product
$\tilde g_i^\top\beta$ simplifies to
\[
  \tilde g_i^\top\beta = (1-\alpha)(x_i^\top\beta)
  + \lambda C_i\underbrace{(v^\top\beta)}_{=\,0} + \xi_i^\top\beta
  = (1-\alpha)(x_i^\top\beta) + q_i.
\]
Using this, $\tilde g_i(\tilde g_i^\top\beta) = [(1-\alpha)x_i + \lambda C_i v + \xi_i][(1-\alpha)(x_i^\top\beta)+q_i]$.
Expanding, averaging over $i$, and noting that $\tfrac{\lambda^2}{2}vv^\top\beta = 0$ since $v\perp\beta$, gives
\begin{align}
  \widehat\Sigma_x\beta
  = \Sigma_x\beta
  + \frac{X^\top q}{(1-\alpha)K}
  + \frac{\lambda(C^\top X\beta)\,v}{(1-\alpha)K}
  + \frac{\lambda(C^\top q)\,v}{(1-\alpha)^2 K}
  &\\\nonumber+ \frac{\Xi^\top X\beta}{(1-\alpha)K}
  + \frac{\Xi^\top q}{(1-\alpha)^2 K}
  - \frac{\sigma_{\mathrm{DP}}^2\beta}{(1-\alpha)^2}
\end{align}
Substituting $\tilde g_i$ into $Z$ gives
\[
  Z = \frac{X^\top Y}{K}
    + \frac{\lambda(C^\top Y)\,v}{(1-\alpha)K}
    + \frac{\Xi^\top Y}{(1-\alpha)K}
\]
Since $\nabla L(\beta) = \Sigma_x\beta - \tfrac{1}{K}X^\top Y$,
the zeroth-order terms in $\widehat\Sigma_x\beta - Z$ cancel $\nabla L(\beta)$ exactly.
Using $C^\top X\beta - C^\top Y = C^\top r$ and $\Xi^\top X\beta - \Xi^\top Y = \Xi^\top r$,
the remaining terms collect as
\[
  G - \nabla L(\beta)
  = \frac{\lambda(C^\top r)\,v + \Xi^\top r + X^\top q}{K(1-\alpha)}
  + \frac{\lambda(C^\top q)\,v + \Xi^\top q - K\sigma_{\mathrm{DP}}^2\beta}{K(1-\alpha)^2}\]
  \[= \frac{T_1+T_2}{K(1-\alpha)} + \frac{T_3}{K(1-\alpha)^2}
\]
Each vector has mean zero: $\E[T_1]=0$ because $\E[C_i]=0$ and $\E[\xi_i]=0$;
$\E[T_2]=0$ because $\E[q_i]=0$; and $\E[T_3]=0$ because
$\E[\Xi^\top q]=K\sigma_{\mathrm{DP}}^2\beta$ (by \ref{lem:moments})
is exactly cancelled by the $-K\sigma_{\mathrm{DP}}^2\beta$ term.
\end{proof}

\section{Scalar Moments}\label{app:moments}

\begin{lemma}[Scalar moments]\label{lem:moments}
For all $i,j$ and any fixed $a\in\R^d$:
\begin{align}
  \E[C_i C_j] &= \tfrac{1}{2}\delta_{ij}, \quad \E[C_i] = 0, \label{eq:Cmom}\\
  \E[q_i q_j] &= \sigma_{\mathrm{DP}}^2\norm{\beta}^2\delta_{ij}, \quad \E[q_i] = 0, \label{eq:qmom}\\
  \E\bigl[q_i(\xi_i^\top a)\bigr] &= \sigma_{\mathrm{DP}}^2(\beta^\top a), \label{eq:cross}\\
  \E\bigl[(\xi_i^\top\beta)\norm{\xi_i}^2\bigr] &= 0, \label{eq:third}\\
  \E\bigl[(\xi_i^\top\beta)^2\norm{\xi_i}^2\bigr]
  &= \sigma_{\mathrm{DP}}^4(d+2)\norm{\beta}^2. \label{eq:fourth}
\end{align}
\end{lemma}

\begin{proof}
Equations \eqref{eq:Cmom} and \eqref{eq:qmom} are standard.
For \eqref{eq:cross}, note that $q_i = \xi_i^\top\beta$, so
\[
  \E[q_i(\xi_i^\top a)] = \sum_{b,c}\beta_b a_c \E[\xi_{ib}\xi_{ic}]
  = \sigma_{\mathrm{DP}}^2 \beta^\top a,
\]
since $\E[\xi_{ib}\xi_{ic}]=\sigma_{\mathrm{DP}}^2\delta_{bc}$.
For \eqref{eq:third}, expand $\E[(\xi_i^\top\beta)\norm{\xi_i}^2] = \sum_{a,b}\beta_b\,\E[\xi_{ib}\xi_{ia}^2]$.
When $b\ne a$ this equals $\E[\xi_{ib}]\E[\xi_{ia}^2]=0$,
and when $b=a$ it equals $\beta_a\E[\xi_{ia}^3]=0$ by Gaussian symmetry.
For \eqref{eq:fourth}, the Isserlis theorem gives
$\E[\xi_{ib}\xi_{ic}\xi_{ia}\xi_{ia}]= \sigma_{\mathrm{DP}}^4(\delta_{bc}+2\delta_{ba}\delta_{ca})$.
Summing over $a$ and then over $b,c$ with weights $\beta_b\beta_c$
yields $\sigma_{\mathrm{DP}}^4(d\norm{\beta}^2+2\norm{\beta}^2)$.
\end{proof}

\section{Block 1: \texorpdfstring{$\E[\norm{T_1+T_2}^2]$}{E[||T1+T2||^2]}}
\label{app:block1}

\begin{lemma}\label{lem:block1}
\[
  \E[\norm{T_1+T_2}^2]
  = \norm{r}^2\!\left(\frac{\lambda^2}{2}+d\sigma_{\mathrm{DP}}^2\right)
  + 2K\sigma_{\mathrm{DP}}^2\,\beta^\top\Sigma_{rx}
  + K\sigma_{\mathrm{DP}}^2\norm{\beta}^2\,\mathrm{tr}(\Sigma_x).
\]
\end{lemma}

\begin{proof}
We use the expansion
$\E[\norm{T_1+T_2}^2] = \E[\norm{T_1}^2]+2\E[T_1^\top T_2]+\E[\norm{T_2}^2]$
and compute each term separately.

\smallskip\noindent\textit{Term $\E[\norm{T_1}^2]$.}
Since $T_1=\lambda(C^\top r)v+\Xi^\top r$ and $\norm{v}=1$,
\[
  \norm{T_1}^2 = \lambda^2(C^\top r)^2 + 2\lambda(C^\top r)(v^\top\Xi^\top r) + \norm{\Xi^\top r}^2.
\]
The first term satisfies $\E[\lambda^2(C^\top r)^2] = \lambda^2\sum_{i,j}r_ir_j\E[C_iC_j]
= \tfrac{\lambda^2}{2}\norm{r}^2$ by \eqref{eq:Cmom}.
The middle term vanishes since $C$ is independent of $\Xi$ and $\E[C]=0$.
Writing $\norm{\Xi^\top r}^2 = \sum_{i,j}r_ir_j(\xi_i^\top\xi_j)$, independence of distinct rows
kills all off-diagonal expectations, while $\E[\norm{\xi_i}^2]=d\sigma_{\mathrm{DP}}^2$ gives
\[
  \E[\norm{\Xi^\top r}^2] = d\sigma_{\mathrm{DP}}^2\norm{r}^2,
\]
so altogether $\E[\norm{T_1}^2] = \norm{r}^2\bigl(\tfrac{\lambda^2}{2}+d\sigma_{\mathrm{DP}}^2\bigr)$.

\smallskip\noindent\textit{Term $\E[T_1^\top T_2]$.}
Expanding,
\[
  T_1^\top T_2 = \lambda(C^\top r)(v^\top X^\top q) + r^\top\Xi X^\top q.
\]
The first summand has zero expectation since $C$ is independent of $q$ and $\E[C]=0$.
Writing the second as $r^\top\Xi X^\top q = \sum_{i,j}r_iq_j(\xi_i^\top x_j)$,
off-diagonal terms ($i\ne j$) vanish by independence, and on the diagonal
\eqref{eq:cross} with $a=x_i$ gives $\E[q_i(\xi_i^\top x_i)]=\sigma_{\mathrm{DP}}^2(\beta^\top x_i)$, so
\[
  \E[T_1^\top T_2] = \sigma_{\mathrm{DP}}^2\sum_i r_i(\beta^\top x_i) = K\sigma_{\mathrm{DP}}^2\,\beta^\top\Sigma_{rx}.
\]

\smallskip\noindent\textit{Term $\E[\norm{T_2}^2]$.}
Since $\norm{T_2}^2 = q^\top XX^\top q$ and $\E[qq^\top] = \sigma_{\mathrm{DP}}^2\norm{\beta}^2 I_K$ by \eqref{eq:qmom},
\[
  \E[\norm{T_2}^2] = \sigma_{\mathrm{DP}}^2\norm{\beta}^2\,\mathrm{tr}(XX^\top)
  = K\sigma_{\mathrm{DP}}^2\norm{\beta}^2\,\mathrm{tr}(\Sigma_x).
\]
Adding the three contributions gives the stated formula.
\end{proof}


\label{app:cross}

\begin{theorem}\label{thm:cross}
$\E[T_1^\top T_3]=0$ and $\E[T_2^\top T_3]=0$.
\end{theorem}

\begin{proof}
\noindent\textit{$\E[T_1^\top T_3]=0$.}
Substituting \eqref{eq:T1} and \eqref{eq:T3} and expanding gives
\begin{align*}
  T_1^\top T_3
  &= \lambda^2(C^\top r)(C^\top q)
  + \lambda(C^\top r)(v^\top\Xi^\top q)
  - \lambda K\sigma_{\mathrm{DP}}^2(C^\top r)(v^\top\beta)\\
  &\quad + \lambda(C^\top q)(r^\top\Xi v)
  + r^\top\Xi\Xi^\top q
  - K\sigma_{\mathrm{DP}}^2\, r^\top\Xi\beta.
\end{align*}
We verify each of the six terms has zero expectation.
The first term factors as $$\lambda^2\sum_{i,j}r_j\E[C_iC_j]\E[q_i]$$ which vanishes
since $C$ and $q$ are independent and $\E[q_i]=0$.
The second and fourth terms both vanish because $C$ is independent of $\Xi$ and $\E[C]=0$.
The third term is zero since $v\perp\beta$.
For the fifth term, write $r^\top\Xi\Xi^\top q = \sum_{i,j}r_iq_j(\xi_i^\top\xi_j)$.
When $i\ne j$ the rows $\xi_i$ and $\xi_j$ are independent, so $\E[\xi_i^\top\xi_j]=0$.
When $i=j$,
\[
  \E\bigl[r_i\,q_i\norm{\xi_i}^2\bigr]
  = r_i\,\E\bigl[(\xi_i^\top\beta)\norm{\xi_i}^2\bigr] = 0
  \qquad\text{by \eqref{eq:third}.}
\]
The sixth term satisfies
$\E[-K\sigma_{\mathrm{DP}}^2\,r^\top\Xi\beta] = -K\sigma_{\mathrm{DP}}^2\sum_i r_i\E[q_i]=0$.

\smallskip\noindent\textit{Similarly, $\E[T_2^\top T_3]=0$.}
Substituting \eqref{eq:T2} and \eqref{eq:T3} and expanding gives
\[
  T_2^\top T_3
  = \lambda(C^\top q)(q^\top Xv)
  + q^\top X\Xi^\top q
  - K\sigma_{\mathrm{DP}}^2\,q^\top X\beta.
\]
The first term vanishes since $C_i$ is independent of all $\xi$-quantities and $\E[C_i]=0$.
For the second, write $q^\top X\Xi^\top q = \sum_{i,j}q_iq_j(x_i^\top\xi_j)$.
When $i\ne j$, $q_i=\xi_i^\top\beta$ is independent of $(\xi_j,q_j)$ and $\E[q_i]=0$, so the term vanishes.
When $i=j$,
\[
  \E\bigl[q_i^2(x_i^\top\xi_i)\bigr]
  = x_i^\top\E\bigl[(\xi_i^\top\beta)^2\xi_i\bigr] = 0
\]
by the same odd-moment argument as \eqref{eq:third}.
The third term satisfies
\[
  \E\bigl[-K\sigma_{\mathrm{DP}}^2\,q^\top X\beta\bigr] = -K\sigma_{\mathrm{DP}}^2(X\beta)^\top\E[q]=0
\] and hence, $\E[(T_1+T_2)^\top T_3]=0$.
\end{proof}

\section{ \texorpdfstring{$\E[\norm{T_3}^2]$}{E[||T3||^2]}}
\label{app:block2}

Write $T_3 = \lambda s\,v + (w - K\sigma_{\mathrm{DP}}^2\beta)$ where
$s:=C^\top q\in\R$ and $w:=\Xi^\top q\in\R^d$.

\begin{lemma}\label{lem:sw}
The scalar $s$ and vector $w$ satisfy
\begin{align}
  \E[s] &= 0, \quad
  \E[s^2] = \frac{K\sigma_{\mathrm{DP}}^2\norm{\beta}^2}{2}, \label{eq:s}\\
  \E\bigl[s(v^\top w)\bigr] &= 0, \label{eq:sw}\\
  \E\bigl[\norm{w - K\sigma_{\mathrm{DP}}^2\beta}^2\bigr]
  &= K\sigma_{\mathrm{DP}}^4\norm{\beta}^2(d+1). \label{eq:w}
\end{align}
\end{lemma}

\begin{proof}
\noindent\textit{Proof of \eqref{eq:s}.}
Since $\E[C_i]=0$ we have $\E[s]=0$.
For $\E[s^2]$, expand $s^2 = \sum_{i,j}C_iC_jq_iq_j$.
Using independence of $C$ and $q$ together with \eqref{eq:Cmom} and \eqref{eq:qmom},
\[
  \E[C_iC_jq_iq_j]
  = \E[C_iC_j]\,\E[q_iq_j]
  = \tfrac{1}{2}\delta_{ij}\cdot\sigma_{\mathrm{DP}}^2\norm{\beta}^2\delta_{ij},
\]
which is nonzero only when $i=j$.  Summing over $i$ gives
$\E[s^2] = \tfrac{K\sigma_{\mathrm{DP}}^2\norm{\beta}^2}{2}$.

\noindent\textit{Proof of \eqref{eq:sw}.}
Write $s(v^\top w) = \sum_{i,j}C_iq_iq_j(\xi_j^\top v)$.
Since each $C_i$ is independent of all $\xi$-quantities and $\E[C_i]=0$, every term vanishes.

\noindent\textit{Proof of \eqref{eq:w}.}
Expand the squared norm as
\[
  \E\bigl[\norm{w - K\sigma_{\mathrm{DP}}^2\beta}^2\bigr]
  = \E[\norm{w}^2]
  - 2K\sigma_{\mathrm{DP}}^2\,\beta^\top\E[w]
  + K^2\sigma_{\mathrm{DP}}^4\norm{\beta}^2.
\]
Applying \eqref{eq:cross} componentwise gives $\E[w] = \E[\Xi^\top q] = K\sigma_{\mathrm{DP}}^2\beta$,
so $\beta^\top\E[w] = K\sigma_{\mathrm{DP}}^2\norm{\beta}^2$.

To compute $\E[\norm{w}^2]$, write $\norm{w}^2 = \sum_{i,j}q_iq_j(\xi_i^\top\xi_j)$ and split
into diagonal and off-diagonal parts.  For $i=j$, \eqref{eq:fourth} gives
\[
  \E\bigl[q_i^2\norm{\xi_i}^2\bigr]
  = \E\bigl[(\xi_i^\top\beta)^2\norm{\xi_i}^2\bigr]
  = \sigma_{\mathrm{DP}}^4(d+2)\norm{\beta}^2,
\]
so the $K$ diagonal terms contribute $K\sigma_{\mathrm{DP}}^4(d+2)\norm{\beta}^2$.
For $i\ne j$, independence of rows gives
\[
  \E\bigl[q_iq_j(\xi_i^\top\xi_j)\bigr]
  = \bigl(\E[q_i\xi_i]\bigr)^\top\!\E[q_j\xi_j]
  = (\sigma_{\mathrm{DP}}^2\beta)^\top(\sigma_{\mathrm{DP}}^2\beta)
  = \sigma_{\mathrm{DP}}^4\norm{\beta}^2,
\]
and there are $K(K-1)$ such pairs.  Adding both contributions,
\[
  \E[\norm{w}^2] = K\sigma_{\mathrm{DP}}^4(d+2)\norm{\beta}^2 + K(K-1)\sigma_{\mathrm{DP}}^4\norm{\beta}^2
  = K\sigma_{\mathrm{DP}}^4(d+K+1)\norm{\beta}^2.
\]
Substituting back,
\[
  \E\bigl[\norm{w-K\sigma_{\mathrm{DP}}^2\beta}^2\bigr]
  = K\sigma_{\mathrm{DP}}^4(d+K+1)\norm{\beta}^2
  - \]\[2K^2\sigma_{\mathrm{DP}}^4\norm{\beta}^2
  + K^2\sigma_{\mathrm{DP}}^4\norm{\beta}^2
  = K\sigma_{\mathrm{DP}}^4(d+1)\norm{\beta}^2.
\]
\end{proof}

\begin{lemma}\label{lem:T32}
$\E[\norm{T_3}^2] = K\norm{\beta}^2\!\left(\dfrac{\lambda^2\sigma_{\mathrm{DP}}^2}{2}+(d+1)\sigma_{\mathrm{DP}}^4\right)$.
\end{lemma}

\begin{proof}
Since $\norm{v}=1$, expanding the squared norm of $T_3 = \lambda s\,v + (w - K\sigma_{\mathrm{DP}}^2\beta)$ gives
\[
  \norm{T_3}^2
  = \lambda^2 s^2
  + 2\lambda s\,v^\top(w - K\sigma_{\mathrm{DP}}^2\beta)
  + \norm{w - K\sigma_{\mathrm{DP}}^2\beta}^2.
\]
Taking expectations of each term in turn:
\begin{align*}
  \E[\lambda^2 s^2]
    &= \frac{\lambda^2 K\sigma_{\mathrm{DP}}^2\norm{\beta}^2}{2}
    && \text{by \eqref{eq:s},}\\[4pt]
  \E\bigl[2\lambda s\,v^\top(w - K\sigma_{\mathrm{DP}}^2\beta)\bigr]
    &= 2\lambda\bigl(\E[s\,v^\top w] - K\sigma_{\mathrm{DP}}^2(v^\top\beta)\,\E[s]\bigr) = 0
    && \text{by \eqref{eq:sw} and $\E[s]=0$,}\\[4pt]
  \E\bigl[\norm{w - K\sigma_{\mathrm{DP}}^2\beta}^2\bigr]
    &= K\sigma_{\mathrm{DP}}^4(d+1)\norm{\beta}^2
    && \text{by \eqref{eq:w}.}
\end{align*}
Adding the three contributions gives the stated formula.
\end{proof}
\begin{tcolorbox}[colback=cyan!10, colframe=black]

\begin{lemma}[Expected Expansion]
Let \( x, x' \in \mathbb{R}^d \) be two points with difference \( h = x - x' \), and let \( v \in \mathbb{R}^d \) be a fixed unit vector (i.e., \( \|v\| = 1 \)). Let \( \phi \sim \operatorname{Unif}[0, 2\pi] \), and for parameters \( \alpha \in (0,1) \), \( \lambda > 0 \), and \( \omega > 0 \). Then the expected squared distance between the transformed points satisfies
\[
\mathbb{E}_\phi \left[ \| g(x) - g(x') \|^2 \right]
=
(1 - \alpha)^2 \| h \|^2 + 2 \lambda^2 \sin^2\left( \frac{\delta}{2} \right),
\]
where \( \delta = \omega \langle h, v \rangle \).
\end{lemma}
\end{tcolorbox}
\begin{proof}
Define \( h = x - x' \) and let
\[
\theta = \omega \langle x, v \rangle + \phi,
\quad
\theta' = \omega \langle x', v \rangle + \phi = \theta - \delta,
\quad
\delta = \omega \langle h, v \rangle.
\]
Then the difference between outputs is
\[
g(x) - g(x') = (1 - \alpha)h + \lambda \left[ \cos(\theta) - \cos(\theta - \delta) \right] v.
\]
Using the identity \( \cos(\theta) - \cos(\theta - \delta) = 2 \sin\left( \frac{\delta}{2} \right) \sin\left( \theta - \frac{\delta}{2} \right) \), we obtain
\[
g(x) - g(x') = (1 - \alpha)h + 2\lambda \sin\left( \frac{\delta}{2} \right) \sin\left( \theta - \frac{\delta}{2} \right) v.
\]
The squared norm is
\[
\|g(x) - g(x')\|^2
= (1 - \alpha)^2 \|h\|^2
+ 4 \lambda^2 \sin^2\left( \frac{\delta}{2} \right) \sin^2\left( \theta - \frac{\delta}{2} \right)
+ \text{cross terms}.
\]
The cross terms vanish in expectation since \( \mathbb{E}_\phi[\sin(\cdot)] = 0 \). Also, \( \theta - \frac{\delta}{2} \sim \operatorname{Unif}[0, 2\pi] \), so
\[
\mathbb{E}_\phi \left[ \sin^2\left( \theta - \frac{\delta}{2} \right) \right] = \frac{1}{2}.
\]
Therefore,
\[
\mathbb{E}_\phi \left[ \|T(x) - T(x')\|^2 \right]
= (1 - \alpha)^2 \|h\|^2 + 2 \lambda^2 \sin^2\left( \frac{\delta}{2} \right).
\]
\end{proof}

\section{Some equivalences to $(\epsilon, \delta)$-differential privacy} \label{appEquiv}
\subsection{Equivalence as a Hypothesis Test}

Given the output of a differentially private mechanism $M$, an attacker's error rates in leaking the privacy (via attacks such as the membership inference attack) are controlled for with differential privacy as follows by equivalently viewing it as a binary hypothesis test of the following kind, with a corresponding control on error rates of the attacker.  
\begin{itemize}
    \item[(i)] \textit{Null Hypothesis ($H_0$)} Attacker observes $M(D)$,  to conclude that Alice is \emph{not} in the dataset. 
    \item[(ii)] \textit{Alternative Hypothesis ($H_1$):} Attacker observes $M(D^\prime)$, to conclude that Alice \textit{is} in the dataset.  
\end{itemize}

The adversary seeks to minimize:
\begin{enumerate}
    \item[(i)] \textit{Type I error:} Incorrectly concluding that Alice is in the dataset when she is not.  
    \item[(ii)] \textit{Type II error:} Failing to detect that Alice is in the dataset when she actually is.  
\end{enumerate}

One is therefore interested in the trade-off function that minimizes Type II error subject to a constraint on the Type I error rate,
\[
T\bigl(\mathcal{M}(D), \mathcal{M}(D')\bigr)(\alpha) \;=\; \inf \{ \text{Type II error} \;\mid\; \text{Type I error} \leq \alpha \}.
\]

If a mechanism reveals information, the adversary can distinguish $M(D)$ from $M(D')$ with improved odds, i.e., achieve lower Type II error at the same Type I error rate. In this sense, the privacy guarantee can be described as a bound on the trade-off curve.  

A mechanism $M$ is $(\varepsilon, \delta)$-differentially private if and only if, for all rejection rules, the following inequalities hold.
\[
\begin{cases}
\Pr(\text{Type I error}) + e^{\varepsilon} \Pr(\text{Type II error}) \;\geq\; 1 - \delta, \\
e^{\varepsilon} \Pr(\text{Type I error}) + \Pr(\text{Type II error}) \;\geq\; 1 - \delta.
\end{cases}
\]

This corresponds to the following trade-off function.
\[
f_{\varepsilon, \delta}(\alpha) \;=\; \max \Bigl\{ 0,\; 1 - \delta - e^{\varepsilon}\alpha,\; e^{-\varepsilon}(1 - \delta - \alpha) \Bigr\}.
\]
\subsubsection{Equivalence as a divergence constraint}
 A mechanism $\mathcal{M}$ is $(\varepsilon, \delta)$-differentially private (or $(\varepsilon, \delta)-D P$ for short) if and only if, for any neighboring input datasets $\mathbf{x}, \mathbf{x}^{\prime}$, it holds that $\mathfrak{D}_{e^{\varepsilon}}\left(\mathcal{M}(\mathbf{x}) \| \mathcal{M}\left(\mathbf{x}^{\prime}\right)\right) \leq \delta$ using the hockey-stick divergence given by $$\mathfrak{D}_{e^{\varepsilon}}\left(\mu \| \mu^{\prime}\right):=\int\left[f_\mu(y)-e^{\varepsilon} \cdot f_{\mu^{\prime}}(y)\right]_{+} d y$$
 For two discrete distributions $\mu_{\mathrm{up}}$ and $\mu_{\mathrm{lo}}$, their privacy loss at $o \in \operatorname{supp}\left(\mu_{\mathrm{up}}\right)$ is defined as

$$
\mathcal{L}_{\mu_{\mathrm{up}} / \mu_{\mathrm{lo}}}(o):=\left\{\begin{array}{ll}
\ln \left(\frac{\mu_{\mathrm{up}}(o)}{\mu_{\mathrm{lo}}(o)}\right) & \text { if } \mu_{\mathrm{lo}}(o)>0 \\
+\infty & \text { if } \mu_{\mathrm{lo}}(o)=0
\end{array} .\right.
$$

The privacy loss distribution (PLD) of $\mu_{\mathrm{up}}$ and $\mu_{\mathrm{lo}}$, denoted by $P L D_{\mu_{\mathrm{up}} / \mu_{\mathrm{lo}}}$, is a distribution on $\mathbb{R} \cup\{\infty\}$ where $y \sim P L D_{\mu_{\mathrm{up}} / \mu_{\mathrm{lo}}}$ is generated as follows: sample $o \sim \mu_{\mathrm{up}}$ and let $y=\mathcal{L}_{\mu_{\mathrm{up}} / \mu_{\mathrm{lo}}}(o)$.
Moreover, for any two distributions $\mu_{u p}$ and $\mu_{l o}$ where both are discrete or both are continuous, it holds that
$$
\mathfrak{D}_{e^{\varepsilon}}\left(\mu_{u p} \| \mu_{l o}\right)=\mathbb{E}_{y \sim PLD_{\mu_{up} / \mu_{lo}}} \left[1-e^{\varepsilon-y}\right]_{+}
$$
Due to this, a mechanism $\mathcal{M}$ is $(\varepsilon, \delta)-D P$ if and only if the following holds for all neighbouring input datasets $\mathbf{x}$ and $\mathbf{x}^{\prime}$ :
$$
\delta \geq \mathbb{E}_{y \sim P L D_{\mathcal{M}(\mathrm{x}) / \mathcal{M}}\left(\mathrm{x}^{\prime}\right)}\left[1-e^{\varepsilon-y}\right]_{+}
$$
 
\section{Linear regression}\label{LRapp}
\subsubsection{Classical results}
In linear regression, the relationship between a response $y \in \mathbb{R}$ and covariates $x \in \mathbb{R}^d$ is given by
$$
y = x^\top \beta + \epsilon, \quad \epsilon \sim \mathcal{N}(0, \sigma^2).
$$
Given $n$ i.i.d. samples $\{(x_i, y_i)\}_{i=1}^n$ from this model, with ${X} \in \mathbb{R}^{n \times d}$ denoting the design matrix and ${Y}$ denoting the response, the ordinary least squares (OLS) estimator of $\beta$ is
$$
\hat{\beta}_{\text{OLS}} = ({X}^\top {X})^{-1} {X}^\top {Y},
$$
and defined when ${X}$ has full column rank. Under Gaussian errors and fixed design, OLS is minimax optimal and achieves the Cramér-Rao lower bound, with the minimax risks characterized as follows.

\begin{enumerate}
    \item \textbf{Prediction error (fixed design)}:
   The minimax risk for the normalized prediction error satisfies
    \[
    \inf_{\hat{\beta}} \sup_{\beta_0 \in \mathbb{R}^d} \mathbb{E}\left[ \frac{1}{n} \| {X} (\hat{\beta} - \beta_0) \|_2^2 \mid {X} \right] = \sigma^2 \frac{\operatorname{rank}({X})}{n}.
    \]
    \item \textbf{Estimation error (fixed design)}:
    \[
    \inf_{\hat{\beta}} \sup_{\beta_0 \in \mathbb{R}^d} \mathbb{E}\left[ \|\hat{\beta} - \beta_0\|_2^2 \mid {X} \right] = \sigma^2 \cdot \operatorname{tr}\left( ({X}^\top {X})^{-1} \right).
    \]
\end{enumerate}

\subsubsection{Statistical learning result}
 Linear regression has also been studied in the distribution free statistical learning setting, where the only assumption is that the data are drawn i.i.d. from some unknown distribution $\mathcal{P}$ defined over some compact domain $\mathcal{X} \times \mathcal{Y}$. Specifically, let the risk be

\[
R(\beta)=\mathbb{E}_{(x, y) \sim \mathcal{P}}\left[\frac{1}{2}\left(x^T \beta-y\right)^2\right].
\]

The minimax risk under this setting when $\beta, \mathcal{X}$ and $\mathcal{Y}$ are Euclidean balls is given \cite{shamir2015sample} by,
$$
\inf _{\hat{\beta}} \sup _{\mathcal{P}}\left[\mathbb{E}[n \cdot R(\hat{\beta})]-\inf _{\beta \in \beta}[n \cdot R(\beta)]\right]
\gtrsim  \min \left\{n\|\mathcal{Y}\|^2,\|\beta\|^2\|\mathcal{X}\|^2+d\|\mathcal{Y}\|^2, \sqrt{n}\|\beta\|\|\mathcal{X}\|\|\mathcal{Y}\|\right\}
$$
where $\hat{\beta}$ be any measurable function of the data set $X, {y}$ to $\beta$ and the expectation is taken over the data generating distribution $(X, {Y}) \sim \mathcal{P}^n$.

\section*{Acknowledgments}
We would like to acknowledge the MBZUAI SU Fund and ADIA Lab Fellowship for supporting this work. We also acknowledge the MIT Institute for Data, Systems, and Society for hosting the primary author.
\printbibliography
\end{document}